\documentclass{article}
\usepackage{kr}

\usepackage{times}
\usepackage{soul}
\usepackage{url}
\usepackage[small]{caption}
\usepackage{graphicx}
\usepackage{amsmath}
\usepackage{amsthm}
\usepackage{booktabs}
\usepackage{xspace}
\usepackage[disable]{todonotes}
\usepackage{framed}
\usepackage{enumitem}
\setlist{noitemsep}
\usepackage{thm-restate}
\usepackage[square]{natbib}
\renewcommand{\cite}[1]{\citep{#1}}
\usepackage{amssymb}

\newcommand{\patrick}[1]{\todo[inline,color=pink]{Patrick: #1}}
\newcommand{\warren}[1]{\todo[inline,color=cyan]{Warren: #1}}
\newcommand{\renate}[1]{\todo[inline,color=orange]{Renate: #1}}
\newcommand{\sophie}[1]{\todo[inline,color=lime]{Sophie: #1}}

\def\showAppendix{Yes, show the appendix}

\title{Signature-Based Abduction for Expressive Description Logics -- Technical Report}

\author{
Patrick Koopmann$^1$\and
Warren Del-Pinto$^2$\and
Sophie Tourret$^3$\And
Renate A. Schmidt$^2$\\
\affiliations
$^1$Institute for Theoretical Computer Science, Technische Universit\"at Dresden, Germany\\
$^2$Department of Computer Science, The University of Manchester, United Kingdom\\
$^3$Automation of Logic Group, Max Planck Institut f\"ur Informatik, Germany
}

\newtheorem{example}{Example}
\newtheorem{definition}{Definition}
\newtheorem{lemma}{Lemma}
\newtheorem{theorem}{Theorem}

\newcommand{\conf}[1]{}
\newcommand{\rep}[1]{{#1}}

\begin{document}

\newcommand{\EL}{\ensuremath{\mathcal{EL}}\xspace}
\newcommand{\ALC}{\ensuremath{\mathcal{ALC}}\xspace}
\newcommand{\ALCnu}{\ensuremath{\mathcal{ALC}\nu}\xspace}
\newcommand{\ALCU}{\ensuremath{\mathcal{ALC}^{\topRole}}\xspace}
\newcommand{\ALCtopRole}{\ensuremath{\mathcal{ALCI}^{\topRole}}\xspace}
\newcommand{\ALCHO}{\ensuremath{\mathcal{ALCHO}}\xspace}
\newcommand{\ALCImu}{\ensuremath{\mathcal{ALCI}\mu}\xspace}
\newcommand{\ALCImuTopRole}{\ensuremath{\mathcal{ALCI}\mu^\topRole}\xspace}

\newcommand{\KB}{\ensuremath{\mathcal{K}}\xspace}
\newcommand{\ont} {\ensuremath{\mathcal{O}}\xspace}
\newcommand{\hyp} {\ensuremath{\mathcal{H}}\xspace}
\newcommand{\fsig} {\ensuremath{\mathcal{F}}\space}
\newcommand{\VF} {\ensuremath{\mathcal{V}}\xspace}
\newcommand{\VR} {\ensuremath{\mathcal{V}^*}\xspace}

\newcommand{\tup}[1]{\langle #1 \rangle}
\newcommand{\reltup}[2]{({#1},{#2})}

\newcommand{\topRole}{\triangledown}

\newcommand{\ND}{\mathsf{N_D}}
\newcommand{\NT}{\mathsf{N_T}}
\newcommand{\Vmc}{\ensuremath{\mathcal{V}}\xspace}
\newcommand{\NV}{\ensuremath{\mathsf{N_V}}\xspace}

\newcommand{\basedefs}{\ensuremath{\mathsf{N_D^\Clauses}}}
\newcommand{\subd}{\ensuremath{\mathbf{D}}}

\newcommand{\cls}{\ensuremath{\varphi}}

\newcommand{\quant}{\mathsf{Q}}
\newcommand{\Sat}[1]{\ensuremath{\mathsf{Sat}(#1)}\xspace}
\newcommand{\Clauses}{\ensuremath{\Phi}\xspace}%

\newcommand{\AllClauses}{\Clauses_a}
\newcommand{\SatS}[2]{\mathsf{Sat}_{#1}(#2)}
\newcommand{\SatSg}[2]{\mathsf{Sat}^g_{#1}(#2)}

\newcommand{\SatClauses}{\ensuremath{\SatS{S}{\Clauses}}}%
\newcommand{\GroundClauses}{\SatSg{S}{\Clauses}}%
\newcommand{\Inds}{\mathbf{I}}

\newcommand{\rewriteto}[2]{[{#1}\hspace{-0.3em}\rightarrow\hspace{-0.3em}{#2}]}

\newcommand{\Wlog}{W.l.o.g.}
\newcommand{\wlg}{w.l.o.g.}

\newcommand{\LETHE}{\textsc{Lethe}\xspace}

\maketitle

\begin{abstract}
Signature-based abduction aims at building hypotheses over a specified set of names, the signature, that explain an observation relative to some background knowledge. %
This type of abduction is useful for tasks such as diagnosis, where the vocabulary used for %
observed symptoms %
differs %
from the vocabulary expected to explain those symptoms.
We present the first complete method solving signature-based abduction for observations 
expressed in the expressive description logic \ALC, which can include TBox and ABox axioms, 
thereby solving the knowledge base abduction problem.
The method is guaranteed to compute a finite and complete set of hypotheses,
and is evaluated on a set of realistic knowledge bases. 

\end{abstract}

\newcommand{\sig}[1]{\mathsf{sig}(#1)}
\newcommand{\cn}[1]{\mathsf{#1}}

\newcommand{\Sig}{\ensuremath{\Sigma}\xspace}
\newcommand{\ALCOImu}{\ensuremath{\mathcal{ALCOI}\mu}\xspace}

\patrick{Some general rules to keep the paper consistent (feel free to add more:

- I would prefer $\Sigma$ as symbol for signatures over $\mathcal{S}$, as the latter always looks like the description logic $\mathcal{S}$ to me.

- please let us use macros for all function names, for instance in $\sig{a}$. We should also make sure that we all use the same macros. Currently, we use for instance two different macros for signatures, abd and Sig.

- when referring to a role instead of a role name, it would be better to use the upper case $R$

}
\section{Introduction}%

Abduction is central to knowledge discovery and 
knowledge processing and has been
intensely studied in artificial intelligence, computer science,
cognitive science, philosophy and logic~\cite{FlachKakas00,KakasALP1992,RayXHAIL2009}.
Abduction is the process of explaining new observations using background knowledge. 
It is an important enabling mechanism for
a variety of tasks that require explanations that go beyond
what is already implied by existing knowledge, including scientific discovery, database update, 
belief expansion, diagnostics, planning,
language interpretation and inductive learning \cite{HobbsEtAl93,FlachKakas00}. %
However, in the description logic (DL) literature, abduction has received much less attention, despite being
recognised as important for ontology repair, query
update and 
matchmaking~\cite{elsenboich2006case,Wei-KleinerDragisicLambrixTBox2014,CalvaneseEtAl13,%
DiNoiaEtAl07}.

In this paper, the first complete and practical approach for a variant of abduction we call 
\emph{signature-based abduction} is presented, which solves this problem for inputs expressed in 
the expressive DL \ALC. In general, the abduction problem considers an 
observation and a knowledge base (KB), and we are looking for an extension of the KB, called 
\emph{hypothesis}, that would logically imply the observation. Without further constraints, this 
definition is too underspecified to be of practical use, as there may be many such extensions, 
including the observation itself. In signature-based abduction, we restrict the space of solutions 
by a set of names, the \emph{signature of abducibles}: any hypothesis has to be constructed using 
only names from this set. This restriction makes sense since in many applications of 
abduction, the vocabulary that we  
expect in an observation differs from the one we would wish for in a helpful explanation.
To illustrate, %
consider the following %
KB~$\KB$: %
\begin{align}
&\cn{EbolaPatient}\equiv\cn{Patient}\sqcap\exists\cn{infectedBy}.\cn{Ebola}\\
&\exists\cn{contactWith}.\cn{EbolaBat}\sqsubseteq
\cn{EbolaPatient}\\
&\cn{EbolaPatient}\sqsubseteq\forall\cn{infected}.\cn{EbolaPatient}\label{align:ex-cyclic}\\
&\cn{EbolaPatient}(p_1)
\end{align}
It states that an $\cn{EbolaPatient}$ is a patient infected by Ebola~(1),
individuals that were in contact with a bat carrying Ebola have Ebola~(2),
individuals infected by an Ebola patient have Ebola~(3), and individual $p_1$ is an Ebola patient~(4).
Suppose it is now observed that patient $p_2$ also has Ebola; that is, we want to explain the 
observation $\cn{EbolaPatient}(p_2)$.
A sufficient extension of $\KB$ to imply this observation is 
$\{\cn{Patient}(p_2)$, $\exists\cn{infectedBy}.\cn{Ebola}(p_2)\}$. However, these axioms do not 
really explain anything. 

To obtain a more meaningful answer, %
our method constrains the
explanations to a specified signature related to known causes of Ebola: 
$$\Sig=\{\cn{EbolaBat},\cn{infected},\cn{contactWith}\}.$$
An explanation using only names from $\Sig$ is for example:%
\begin{align*}
 \exists\cn{contactWith}.\cn{EbolaBat}(p_2)
 \vee
 \cn{infected}(p_1,p_2),
\end{align*}
stating that either $p_2$ had contact with an Ebola bat, or $p_1$ infected $p_2$.
\newcommand{\Omc}{\ensuremath{\mathcal{O}}\xspace}%
Note that we allow names in the signature to be arbitrarily combined using the constructs of the 
language. Furthermore, we allow for the use of \emph{disjunction}, which in our setting does not 
trivialise the problem, but allows for capturing more than one hypothesis in one solution of the 
abduction problem. Specifically, we are interested in a hypothesis that generalises every possible 
hypothesis, and is thus semantically minimal among all possible solutions. 
Ideally, we would present 
the optimal hypothesis as a disjunction of hypotheses that are \emph{independent} of each other, 
in the sense that there are no logical relations between the hypotheses. This is not always 
possible in \ALC, even when extended with disjunctions: as we show, in general, also  
nominals, inverse roles and fixpoints can be required by such a solution. 

\newcommand{\Amc}{\ensuremath{\mathcal{A}}\xspace}
Constraining hypotheses to those using only a set of allowed names, called \emph{abducibles}, is a 
long-standing practice in 
abductive logic programming~\cite{KakasALP1992,RayXHAIL2009}, and has recently also been 
investigated for first-order logic~\cite{echenim2017prime,echenim2018generic}.
In the domain of DL knowledge bases, most research either does not consider abducibles, or 
only in restricted forms. 

\citet{elsenboich2006case} motivate abduction in DL ontologies through several use cases, give 
desirable properties of a sensible abduction operator, and distinguish between \emph{ABox, TBox and 
knowledge base (KB) abduction}. 
The example above is an \emph{ABox abduction} problem, where the observation
and the hypotheses consist of facts about specific individuals.
Methods for ABox abduction without signature-restriction were investigated by 
\citet{KlarmanABox2011} and \citet{HallandBritzABox2012}, where the hypotheses computed cannot use 
concept disjunction and
negation is restricted to concept names.
Based on these methods, \citet{PukancovaHomola2017,PukancovaHomola2018} proposed an approach for 
ABox abduction %
relying on minimal hitting sets. %
In ontology repair, often it is interesting to explain not only facts (ABox assertions), but also 
terminological axioms (TBox axioms). %
\emph{TBox abduction} has been studied by \citet{Wei-KleinerDragisicLambrixTBox2014} and 
\citet{du2017practical}, though %
not the signature-based variant. 

Abducibles  were considered in~\citet{bienvenu2008complexity} %
for
the light-weight DL $\EL$, but there hypotheses can only be conjunctions over those abducibles, 
and more complex combinations are not possible. In~\cite{CalvaneseEtAl13}, signature-based ABox 
abduction for DL-Lite is investigated theoretically. Different to our work, solutions consider only 
flat ABoxes without complex concepts, but may introduce fresh individual names.
The first work that uses signature-based abduction as considered
here is by \citet{Del-PintoSchmidt2019} for ABox abduction in \ALC,
but with some restrictions: i) the observation can only consider a
single individual,
and ii) the set of abducibles has to contain all role names.
Our approach allows both TBoxes, ABoxes and mixed observations, and poses no restrictions on the 
signature, and is thus the first work that solves signature-based abduction in the most general 
form. This means that both observations and explanations can be in mixed form. Consider the knowledge base 
$\{A \sqsubseteq \forall r.B, \lnot C(b)\}$, 
observation $\forall r.C(a)$ and signature $\Sig = \{A, B, C\}$.
The solution returned by the approach in this paper is $\{A(a),
B \sqsubseteq C\}$ containing both an ABox and a TBox axiom.
Finding abduction solutions is often implemented as inverse deduction.
While deduction infers consequences from a given set of
premises, abduction infers missing premises from
which the input observations are consequences relative to some background knowledge. 
In signature-based abduction, the aim is finding the \emph{most general
set of hypotheses} over the given signature that entails the input
observations.
This is the reverse of looking for the \emph{most specific consequences} of 
the input over a given signature, which is called the
\emph{uniform interpolation problem}. 
Since these are dual notions~\cite{Lin2001}, in principle abduction
and uniform interpolation problems
reduce to each other via negation.
Specifically, for a KB $\KB$ and observation~$\Psi$, we may perform signature-based abduction by 
computing a 
uniform interpolant for $\KB\wedge\neg\Psi$.
This idea has been used for signature-based abduction
in first-order logic using second-order quantifier
elimination~\cite{DohertyLukasSzalas2001,GabbaySchmidtSzalas2008},
a concept closely related to uniform interpolation.
However, so far complete methods for uniform
interpolation only exist for DLs that are not closed under
negation~\cite{KoopmannSchmidt2014,FORGETTING-YIZHENG,FOUNDATIONS-EXPRESSIVE-UI}.
Hence we cannot directly use existing approaches to uniform
interpolation, but must develop new methods that can deal
with negated KBs. This is why the method by~\citet{Del-PintoSchmidt2019}, which uses the 
existing uniform interpolation tool \LETHE~\cite{LETHE}, only supports a restricted abduction 
setting.

It turns out the seemingly small extension of adding negation brings significant new
challenges to uniform interpolation, which can already be seen from the fact that abduction 
solutions may require a more 
expressive DL (disjunctive $\ALCOImu$) than uniform interpolants (disjunctive 
$\ALCnu$)~\cite{KoopmannSchmidt2014}. 
But even with a uniform interpolation method that solves this issue, there remains a practical 
problem. Uniform interpolation for \ALC is challenging, as solutions can in the 
worst case be of size triple-exponential in the size of the input~\cite{FOUNDATIONS-EXPRESSIVE-UI}. 
This challenge does impact practical implementations, which usually only perform well with 
signatures that are either very small or very large~\cite{KoopmannSchmidt2014,FORGETTING-YIZHENG,ChenAlghamdiEtAl19a}. Moreover, as we would use 
$\KB\wedge\neg\Psi$ as 
input, most of the computed consequences would only depend on $\KB$ and thus have no relevance to 
the abduction problem of explaining $\Psi$.
In practice, we expect $\Psi$ to be considerably smaller than $\KB$, so that the majority of 
computed consequences would be useless.
In~\cite{Del-PintoSchmidt2019}, irrelevant consequences are 
removed using 
a filtering
technique and post-processing. A more efficient solution, essential for larger KBs,
is to not compute irrelevant consequences in the first place.

In this paper, we solve the signature-based abduction problem for \ALC using ideas from the uniform 
interpolation method for \ALC KBs presented in~\cite{KoopmannSchmidt2014}. This method is 
generalised to deal with Boolean \ALC KBs as input, using a modification of 
the \emph{set-of-support strategy}~\cite{plaisted1994search} to prune irrelevant inferences. The 
method by~\citet{KoopmannSchmidt2014} uses a resolution-based calculus to compute relevant 
inferences for a given signature. Specific to this approach is that new concept names are introduced 
during the process, so-called definers, which are eliminated using simple unfolding operations in 
a final post-processing step. In our setting, the problem becomes more challenging, leading to a 
more complex resolution calculus and a more involved definer elimination step during which inverse 
roles and nominals are introduced. The extended calculus and the generalised form of the input lead 
to special challenges also in the implementation, which for practicality has to determine required 
rule applications in a goal-oriented way. Further post-processing is applied to compute abduction 
solutions that are not only semantically minimal, but are also presented in the natural form of a 
disjunction of alternative hypotheses, which are \emph{independent} in the sense that no hypothesis 
is logically implied by another. 

To summarise, the contributions of this paper are:
We solve the signature-based KB abduction problem for \ALC in the most general form, allowing for 
arbitrary signatures, as well as KBs, as input. We establish a minimal extension of \ALC that is 
sufficient to cover the set of all possible solutions to a signature-based abduction problem in a 
single disjunction of KBs. We extend the resolution-calculus in~\cite{KoopmannSchmidt2014} to deal 
with Boolean \ALC-KBs. We devise a modification of the set-of-support strategy that can be used 
with this calculus. We develop a method to compute solutions to the signature-based abduction 
problem that are semantically minimal in the sense considered here. 
We introduce techniques to avoid irrelevant explanations and remove dependencies between 
explanations.
We implemented and evaluated the method on a corpus of realistic
ontologies, finding that solutions can be computed in practice and
usually consist of few small hypotheses.

\conf{
Additional details and proofs are provided in the extended version of the paper~\cite{TR}.
}
\section{Description Logic Knowledge Bases}%
\label{sec:preliminaries}

\newcommand{\Kmc}{\ensuremath{\mathcal{K}}\xspace}
\newcommand{\Imc}{\ensuremath{\mathcal{I}}\xspace}
\newcommand{\NC}{\ensuremath{\mathsf{N_C}}\xspace}
\newcommand{\NR}{\ensuremath{\mathsf{N_R}}\xspace}
\newcommand{\NI}{\ensuremath{\mathsf{N_I}}\xspace}

\newcommand{\ALCOImuTopRole}{\ensuremath{\ALCOImu^\topRole}\xspace}

We recall the DLs \ALC, \ALCOImu, and \ALCOImuTopRole relevant to this paper~\cite{DL-HANDBOOK}. Let $\NV$, $\NC$, $\NR$ and $\NI$ be pair-wise disjoint, countably infinite sets of respectively \emph{concept variable-}, \emph{concept}-, \emph{role}- and \emph{individual names}. A \emph{role} is an expression of the form $r$, $r^-$ (inverse role) or $\topRole$ (universal role), where $r\in\NR$. \ALCOImuTopRole concepts are then built according to the following syntax rule, where $A\in\NC$, $X\in\NV$, $a\in\NI$ and $R$ is a role:
\[
 C ::= A \mid \{a\} \mid X \mid \neg C \mid C\sqcup C \mid \exists R.C \mid \mu X.C.
\]
We additionally require that for \emph{least fixpoint concepts} $\mu X.C$, $X$ occurs in $C$ only 
under an even number of negations ($\neg$). This is a standard requirement 
to ensure that the semantics of the least fixpoint concept is always 
well-defined~\cite{CalvaneseGL99}.
We use $C_1[C_2\mapsto C_3]$ to denote the result of replacing in $C_1$ every $C_2$ by $C_3$. 
Further concepts are introduced as abbreviations: $\top=(A\sqcup\neg A)$ (for an arbitrary $A\in\NC$), $\bot=\neg\top$, $C_1\sqcap C_2=\neg(\neg C_1\sqcup\neg C_2)$, $\forall R.C=\neg\exists R.\neg C$ and $\nu X.C[X]=\neg \mu X.\neg C[X\mapsto\neg X]$. 

A concept is \emph{closed} if every concept variable name $X$ occurs in the scope of the least 
fixpoint operator $\mu X$. \emph{Knowledge bases} are sets of \emph{concept inclusions} (CIs) of the 
form $C_1\sqsubseteq C_2$ and \emph{assertions} of the form $C_1(a)$, $r(a,b)$, where $C_1$, $C_2$ 
are closed concepts, $a,b\in\NI$ and $r\in\NR$. CIs and assertions are collectively called 
\emph{axioms}. 

A \emph{Boolean KB} is built according to the syntax rule $$\Kmc ::= \alpha \mid 
\neg\Kmc \mid \Kmc\wedge\Kmc \mid \Kmc\vee\Kmc,$$
where $\alpha$ is an axiom. We identify each KB with the Boolean KB that is the conjunction of all 
its axioms. 
If a (Boolean) KB/concept does not use the universal role, it is in \ALCOImu, if it does not use 
inverse roles, fixpoint operators and \emph{nominals} $\{a\}$, it is in \ALC. For an expression 
$E$, 
$\sig{E}$ denotes the concept and role names occurring in $E$. For a signature 
$\Sig\subseteq\NC\cup\NR$, a \emph{$\Sig$-axiom} is an axiom $\alpha$ with 
$\sig{\alpha}\subseteq\Sig$.

The semantics is defined in terms of \emph{interpretations} $\Imc=\tup{\Delta^\Imc,\cdot^\Imc, 
\cdot^{\Imc,\cdot}}$, with $\Delta^\Imc$ a set of \emph{domain elements} and $\cdot^\Imc$ the 
\emph{interpretation function} mapping every $a\in\NI$ to some $a^\Imc\in\Delta^\Imc$, every 
$A\in\NC$ to some $A^\Imc\subseteq\Delta^\Imc$, every $r\in\NR$ to some 
$r\subseteq\Delta^\Imc\times\Delta^\Imc$, and is extended to inverse and universal roles by $(r^-)^\Imc=(r^\Imc)^-$ and 
$\topRole^\Imc=\Delta^\Imc\times\Delta^\Imc$. A \emph{valuation for $\Imc$} is a 
function $\Vmc:\NV\rightarrow 2^{\Delta^\Imc}$. Given a valuation $\Vmc$ for $\Imc$, $X\in\NV$ and 
$W\subseteq\Delta^\Imc$, $\Vmc[X\mapsto W]$ is the valuation identical to $\Vmc$ except $\Vmc(X)= 
W$. Concepts $C$ and valuations $\Vmc$ are mapped to subsets of $C^{\Imc,\Vmc}\subseteq\Delta^\Imc$ 
by 
\begin{align*}
 X^{\Imc,\Vmc}            &= \Vmc(X), \quad
 A^{\Imc,\Vmc}             =A^\Imc, \quad 
 \{a\}^{\Imc,\Vmc} = \{a^\Imc\}, \\
 (\neg C)^{\Imc,\Vmc}      &= \Delta^\Imc\setminus C^{\Imc,\Vmc}, \quad
 (C\sqcap D)^{\Imc,\Vmc}   = C^{\Imc,\Vmc}\cap D^{\Imc,\Vmc}, \\
 (\exists R.C)^{\Imc,\Vmc} &= \{d\in\Delta^\Imc\mid \exists \reltup{d}{e}\in R^\Imc : e\in C^{\Imc,\Vmc}\},\text{ and} \\
 (\mu X.C)^{\Imc,\Vmc}     &= 
 \bigcap\{W\subseteq\Delta^\Imc \mid C^{\Imc,\Vmc[X\mapsto W]}\subseteq W\} .
\end{align*}
Intuitively, concepts $\mu X.C[X]$ are equivalent to an unbounded disjunction of concepts:
\[
 C[\bot]\sqcup C[C[\bot]] \sqcup C[C[C[\bot]]] \sqcup C[C[C[C[\bot]]]]\sqcup\ldots .
\]
$(\mu X.C)^{\Imc,\Vmc}$ is independent of the value of $\Vmc(X)$. Thus, for closed concepts $C$, 
$C^{\Imc,\Vmc}$ is independent of $\Vmc$. We extend~$\cdot^\Imc$ to closed concepts by setting 
$C^\Imc=C^{\Imc,\Vmc}$ for an arbitrary $\Vmc$.

We define \emph{satisfaction of Boolean KBs $\Kmc$ in \Imc}, in symbols $\Imc\models\Kmc$, by $\Imc\models C\sqsubseteq D$ iff $C^\Imc\subseteq D^\Imc$, $\Imc\models C(a)$ if $a^\Imc\in C^\Imc$, $\Imc\models r(a,b)$ if $\reltup{a}{b}\in r^\Imc$, $\Imc\models\neg\Kmc$ if $\Imc\not\models\Kmc$, $\Imc\models\Kmc_1\wedge\Kmc_2$ if $\Imc\models\Kmc_1$ and $\Imc\models\Kmc_2$, and $\Imc\models\Kmc_1\vee\Kmc_2$ if $\Imc\models\Kmc_1$ or $\Imc\models\Kmc_2$. We then also say $\Imc$ is a model of $\Kmc$. We write $\Kmc_1\models\Kmc_2$ if every model of $\Kmc_1$ is a model of $\Kmc_2$.

\section{Signature-Based Abduction}%
\label{sec:abduction-problem}

Our aim is to produce the least assumptive hypothesis within a given signature of allowed \emph{abducible} names. 

\begin{definition}\label{KBAbductionProblem}
Let $\Kmc$ be an $\ALC$ KB (the \emph{background knowledge}), $\Psi$ a set of CIs and assertions in \ALC (the \emph{observation}), and 
$\Sig$ a signature (the set of \emph{abducibles}).
The \emph{signature-based abduction problem} $\tup{\Kmc,\Psi,\Sig}$ is then to compute a hypothesis $\hyp = \bigvee_{i=1}^n \KB_i$ %
that satisfies all of the following conditions:
\begin{enumerate}[%
label=\textbf{A\arabic*},leftmargin=*]
\item\label{en:consistent} $\KB\wedge \hyp \not\models \perp$,
\item\label{en:explains} $\KB\wedge \hyp \models \Psi$,
\item\label{en:signature} $\sig{\hyp} \subseteq \Sig$,
\item\label{en:semantic-minimality} for any \ALC KB $\hyp'$ satisfying~\ref{en:explains} and~\ref{en:signature}, %
$\KB\wedge\hyp'\models\hyp$.
\end{enumerate}
If $\hyp$ satisfies Conditions~\ref{en:consistent}, ~\ref{en:explains} and~\ref{en:signature}, 
it is a  \emph{hypothesis for $\Psi$ in $\Sig$}. If it additionally satisfies 
Condition~\ref{en:semantic-minimality}, it is an \emph{optimal hypothesis}.
\end{definition}
\patrick{Splitting the definition makes it easier in the following examples. I think it is possible to reformulate the definition also to make it shorter.}
Conditions~\ref{en:consistent} and ~\ref{en:explains} are the standard conditions for the 
abduction problem: the hypothesis should be consistent with the background knowledge and 
should 
\emph{explain} the observation. Condition~\ref{en:signature} is what makes this a 
\emph{signature-based} abduction problem. Condition~\ref{en:semantic-minimality} finally 
requires 
the solution to be \emph{optimal}, in the sense that every possible explanation in \ALC is covered.
This is often captured by the notion of \emph{semantic minimality}: \hyp should be the most general hypothesis, the one that makes the least assumptions, among all possible hypotheses.
Represented as a disjunction, \hyp can be seen as a collection of possible hypotheses for the observation.
The notion of semantic minimality is often considered without disjunctions 
\cite{KlarmanABox2011,HallandBritzABox2012}. Note that in the current case, allowing for disjunctions 
does not trivialise the problem, as the signature of abducibles restricts the space of solutions. 
The problem of finding a hypothesis using a \emph{minimal} number of \emph{independent} solutions, 
i.e., using only disjunctions where necessary, is considered in Sect.~\ref{sec:redundancy}.

It is possible that the only KBs that satisfy 
Conditions~\ref{en:explains}--\ref{en:signature} are inconsistent. In this case, due to 
Condition~\ref{en:consistent}, the signature-based abduction problem has no 
solution. In all other cases, an optimal hypothesis exists and is computed by our method. 
Condition~\ref{en:semantic-minimality} ensures that, if there is a hypothesis, there is also an 
optimal solution unique up to logical equivalence.

Condition~\ref{en:semantic-minimality} only refers to hypotheses in \ALC, and both observation and 
explanation are expressed in \ALC. However, for a single hypothesis to cover every possible 
hypothesis, further expressivity might be needed, as the set of possible hypotheses might be 
infinite. The first reason is that additional individuals may play a role in a hypothesis, the 
second reason is that the ontology may entail cyclic CIs. As a consequence, optimal hypotheses might 
need the use of inverse roles, nominals, or fixpoint expressions.
We illustrate the need for further expressivity with the example from the introduction. 

\paragraph{Inverse Roles.} For every individual name $a\in\NI$, the following KB is a hypothesis 
in \ALC for the observation $\cn{EbolaPatient}(p_2)$ :
\[
 \{\ \exists\cn{contactWith}.\cn{EbolaBat}(a), \  
 \cn{infected}(a,p_2)\ \}.
\]
This makes one hypothesis for each individual in the countably infinite set~$\NI$. 
One way around this problem could be to alter the requirements in Def.~\ref{KBAbductionProblem}, by 
enforcing individual names to be taken from a finite set.
But then we would have a large number of solutions that are all identical modulo renaming of 
individual names, 
which is neither convenient nor insightful. 
Another solution is to use inverse roles, which allow coverage of all of the above hypotheses:
\[
 \big(\exists\cn{infected}^-.\exists\cn{contactWith}.\cn{EbolaBat}\big)(p_2).
\]

\paragraph{Nominals.} 
A similar problem 
occurs 
when additional 
individuals can connect two individuals from the observation and background knowledge. Hypotheses 
such as
\[
 \{\ \cn{infected}(p_1,a),\ \cn{infected}(a,p_2) \},
\]
where $a\in\NI$, cannot be captured using inverse roles alone, but can using nominals:
\[
 \big(\exists\cn{infected}.\exists\cn{infected}.\{p_2\}\big)(p_1).
\]

\paragraph{Least Fixpoints.} Finally, due to the cyclic axiom~\eqref{align:ex-cyclic}, any chain of 
infections connecting a known Ebola patient to $p_2$ is a valid hypothesis. To cover all of these
unbounded chains,
we use the least fixpoint operator. %
Using all three constructs, we can represent the optimal hypothesis as follows:
\begin{align*}
\mu X.\big(&\exists\cn{contactWith}.\cn{EbolaBat}\\
{}&\sqcup \exists\cn{infected}^-.\{p_1\}\sqcup\exists\cn{infected}^-.X\big)(p_2),
\end{align*}
which could be read as: ``\emph{One the following happened: 1)~$p_2$~was in contact with an Ebola carrying bat, 2) $p_2$ was infected by $p_1$, or 3) $p_2$ was infected by someone else to whom Conditions~1--3 apply.}''
\patrick{Is this readable or should we formulate this differently?}

From an optimal hypothesis with least fixpoints, explanations without fixpoints can be easily 
obtained by unfolding. Furthermore, inverse roles and nominals are only needed under existential 
role restrictions, so that we can often reconstruct the hypotheses without those constructs using 
additional individual names. We note that while least fixpoints are not supported by standard DL 
reasoners, satisfiability of KBs with \emph{greatest} fixpoints can be decided by using auxiliary 
concept names~\cite{KoopmannSchmidt2014}. 
The same technique can be used to decide entailment of KBs with least fixpoints.

\section{Computing Optimal Hypotheses} %
\label{sec:forgetting-method}

\renate{The difference to the previous approach is not only the presence of universal roles?
Can we describe more precisely the formulae that are manipulated by the
new calculus?}

The general idea for solving abduction problems $\tup{\Kmc,\Psi,\Sigma}$ is to consider the Boolean 
KB $\Kmc\wedge\neg\Psi$, and eliminate the names outside of $\Sigma$ similar to how it is done in 
the resolution-based method for uniform interpolation for \ALC-KBs presented 
in~\cite{KoopmannSchmidt2014}. The resulting KB can then be used to construct a hypothesis by 
negating it again---if this hypothesis is consistent with $\Kmc$, it is an optimal abductive 
explanation, and otherwise, such a hypothesis does not exist.

In uniform interpolation, the aim is to compute, for a given KB~$\Kmc$ and signature $\Sigma$, a KB 
$\Kmc_\Sigma$ such that $\sig{\Kmc_\Sigma}\subseteq\Sigma$ and for 
every $\Sigma$-axiom 
$\alpha$, $\Kmc\models\alpha$ iff $\Kmc_\Sigma\models\alpha$. This KB~$\Kmc_\Sigma$ is computed 
in~\cite{KoopmannSchmidt2014} by eliminating each name in $\sig{\Kmc}\setminus\Sigma$ one after the 
other. To eliminate a name, a set of rules is used to perform all relevant inferences on that name. 
During this process, new names called \emph{definers} are introduced. After a name 
$\sig{\Kmc}\setminus\Sigma$ has been successfully eliminated, these definers are eliminated again
using simple rewriting rules. 

While the general structure of our method is similar, there are several additional challenges we 
have to address.
\begin{enumerate}
 \item In order to support \emph{Boolean} KBs, the calculus in~\cite{KoopmannSchmidt2014} has to be supplemented with further rules.
\item
Existing methods for uniform interpolation perform best when the
signature or its complement relative to the signature of the ontology are
small~\cite{KoopmannSchmidt2014,FORGETTING-YIZHENG,ChenAlghamdiEtAl19a}. For signature-based
abduction, we exploit the fact that the observation $\Psi$ is usually
small compared to $\Kmc$. Our method focuses on inferences relevant to
$\neg\Psi$ using a modified set-of-support strategy. Effectively, this
means we eliminate names in $\neg\Psi$, but not in $\Kmc$.
 \item In uniform interpolation, we can eliminate names one after the other and do not have to 
reconsider previously eliminated names, which means we can encapsulate the elimination of each name
without affecting termination. For abduction, this approach would cause a termination problem, 
because nothing is eliminated in $\Kmc$, and previously eliminated names might 
get propagated back into the hypothesis. Thus, a more integrated approach is needed.
 \item Eliminating definers turns out to be more challenging in our setting, and a simple 
set of rewriting rules is not sufficient anymore. In fact, it is only in this step that nominals 
and inverse roles are introduced. For uniform interpolation of $\ALC$-KBs, disjunction and 
fixpoints are the only required language extension.
 \item Finally, we want to compute hypotheses that have the form of a disjunction of alternative 
hypotheses in $\ALCOImu$. It turns out that for the calculus, it is easier to compute the optimal 
hypothesis first in $\ALCOImuTopRole$ and then reformulate it into a 
disjunction of $\ALCOImu$-KBs.
\end{enumerate}
 Our method to solve 
 the abduction problem $\tup{\Kmc,\Psi,\Sig}$ proceeds using four steps, which we describe in turn below.
\begin{enumerate}[left=0pt,%
label=\textbf{Step \arabic*},leftmargin=*]%
 \item Transform $\Kmc\wedge\neg\Psi$ into clausal normal form.
\label{step_one}
 \item Eliminate all names outside $\Sig$ using the calculus.
\label{step_two}
 \item Express the resulting clause set as a Boolean KB. 
\label{step_three}
 \item Negate the result and eliminate universal roles.
\label{step_four}
\end{enumerate}

\subsection{Step 1: Normalisation}
Our method uses the following clausal normal form.
\begin{definition}[Normal form]
\label{def:normform}
 Let $\ND\subset\NC$ be a special set of concept names, called \emph{definers}, and 
$\NT=\NI\cup\{x\}$ be the set of \emph{terms}, where $x$ is a fresh symbol referring to a 
universally quantified variable. 
\emph{Literals} are built according to the following syntax rule:
\[
 L::= A(t) \mid \neg A(t) \mid \quant r.D(t)\mid r(a,b) \mid \neg r(a,b),
\]
where $t\in\NT$, $r\in\NR\cup\{\topRole\}$, $a,b\in\NI$, $A\in\NC$, $\quant\in\{\exists,\forall\}$ and $D\in\ND$. 
A \emph{clause} $\cls$ is an expression of the form $L_1\vee\ldots\vee L_n$, where each $L_i$ is a literal and at most one literal is of the form $\neg D(x)$ where $D\in\ND$.
We treat clauses as sets of literals by ignoring the order of the literals in a clause 
and 
assuming that clauses contain no duplicates.
\end{definition}
A similar normal form is used in~\cite{KoopmannSchmidt2014}, however without 
universal roles, negated role assertions, or clauses that mix role assertions with other literals.
Clauses $L_1(x)\vee\ldots\vee L_n(x)$ are interpreted as $\top\sqsubseteq L_1\sqcup\ldots\sqcup L_n$, and clauses without variables as disjunctions of ABox assertions.
Our method never introduces clauses that mix variable and constant terms.

To make the following more convenient, we also allow clause sets and Boolean KBs to be mixed: 
Specifically, a generalised KB is a set $\Clauses\cup\{\Kmc\}$ containing a set \Clauses of clauses 
and a Boolean KB~$\Kmc$, and an interpretation $\Imc$ is a model of~$\Clauses\cup\{\Kmc\}$ if both $\Imc\models\Kmc$ 
and \mbox{$\Imc\models\Clauses$}. Entailment of axioms from generalised KBs is defined accordingly.

In the rest of this section, we adopt the following naming conventions, where additional primes, 
sub- or superscripts may be used:
$a,b\in\NI$;
$D\in\ND$;
$t\in\NT$;
$A,B\in\NC$;
$r\in\NR\cup\{\topRole\}$;
$\quant \in \{\exists,\forall\}$;
$C$ is a concept;
$L(t)$ is a literal with argument $t$; and
$\cls$ is a clause.
A \emph{definer-free} Boolean KB \Kmc is such that $\sig{\Kmc}\cap\ND = \emptyset$.

We describe how $\Kmc\wedge\neg\Psi$ is turned into a set of clauses that
preserves all definer-free entailments.
The observation~$\Psi$ is a conjunction of assertions and CIs, and the negated CI 
\mbox{$\neg(C\sqsubseteq D)$} is equivalent to $(\exists\topRole.(C\sqcap\neg D))(a)$, where~$a$ 
can be any individual.
Hence, $\neg\Psi$ can be equivalently represented as a disjunction of negated
assertions. Using standard logical laws and the equivalence between $C_1\sqsubseteq C_2$ and 
$\top\sqsubseteq\neg C_1\sqcup C_2$, we ensure that every CI has $\top$ on
the left-hand side and negation occurs only in front of concept names. Definers
are then used to replace concepts under role restrictions: every $\quant r.C$
gets replaced by $\quant r.D_C$, and we add the CI $\top\sqsubseteq \neg
D_C\sqcup C$, where $D_C$ is a definer uniquely associated with $C$. 
Using standard CNF transformations we can then make sure that
every CI is of the form $\top\sqsubseteq L_1\sqcup\ldots\sqcup L_n$, equivalent
to the clause $L_1(x)\vee\ldots\vee L_n(x)$, and that the negated observation
corresponds to a set of clauses without variables.

\begin{example}\label{ex:normalise}
Let $\Kmc=\{A_1\sqsubseteq B, B\sqsubseteq \exists r.B\}$, 
$\Psi=\{B(a), A_2\sqsubseteq \exists r.B\}$ and $\tup{\Kmc,\Psi,\Sig}$ be an abduction problem for 
some \Sig.
 We first represent $\neg\Psi$ as $\neg B(a)\vee\exists\topRole.(A_2\sqcap\neg\exists 
r.B)(a)$. 
 After normalisation, we obtain $\Clauses=\mbox{$\Clauses_\Kmc\cup\Clauses_{\neg\Psi}$}$, where 
$\Clauses_\Kmc=\{
    \neg A_1(x)\vee B(x)$, 
    \mbox{$\neg B(x)\vee\exists r.D_1(x)$}, 
    \mbox{$\neg D_1(x)\vee B(x)\}$} and 
 $\Clauses_{\neg\Psi}=\{
    \neg B(a)\vee\exists \topRole.D_2(a)$, 
    $\neg D_2(x)\vee A_2(x)$,
    $\neg D_2(x)\vee\forall r.D_3(x)$, $\neg D_3(x)\vee\neg B(x)\}$.
\end{example}

\subsection{Step 2: Elimination of Names Outside $\Sig$}
\newcommand{\SupportClauses}{\ensuremath{\Clauses_s}\xspace}

Names outside of $\Sig$ are eliminated using a resolution procedure with a special set-of-support 
strategy that 
is based on the calculus in Fig.~\ref{fig:calculus}. We first describe these rules and then discuss how 
they are used.
Rules R$A$, R$\quant$, R$\forall$-1 and R$\topRole$ are taken as is, or adapted, from the calculus 
by~\citet{KoopmannSchmidt2014} for \ALC KBs.
The other rules are new.

\begin{figure}
  \begin{framed}
 \begin{tabular}{l c }
  R$A$ &
$\dfrac{ \cls_1\vee A(t_1) \quad \cls_2 \vee \neg A(t_2)
 }{
 (\cls_1 \vee \cls_2)\sigma
 }$ 
\\
\\
R$r$ & 
$\dfrac{ \cls_1\vee r(a,b)   \quad   \cls_2\vee \neg r(a,b) }
 {\cls_1 \vee \cls_2}$
\\
\\
 R$\quant$ &
$
             \dfrac{ \cls_1 \vee (\quant r.D_{\subd_1})(t_1)  \quad \cls_2\vee (\forall r.D_{\subd_2})(t_2)
 }{
     (\cls_1\vee \cls_2\vee\quant r.D_{\subd_1\cup\subd_2}(t_1))\sigma %
 }$
\\
\\
R$\forall$-1 &
$\dfrac{ \cls_1\vee r(t_1,b)  \quad \cls_2\vee (\forall r.D)(t_2)
 }{
 (\cls_1\vee \cls_2\vee D(b))\sigma
 }
$
\\
\\
R$\forall$-2 &
$\dfrac{ \cls_1\vee \neg D(a)  \quad \cls_2\vee (\forall r.D)(b)
}{
	\cls_1\vee \cls_2\vee \neg r(b,a)
}
$
\\
\\
R$\exists$ &
$\dfrac{ \cls_1\vee(\exists r.D)(t) 
 }{
 \cls_1\vee(\exists\topRole.D)(t)
 }$
 \\
 \\
R$\topRole$ &
$\dfrac{\cls\vee(\exists\topRole.D)(t)\quad \neg D(x)}{\cls}$
 \end{tabular}
 \end{framed}

 \caption{Calculus for eliminating concept and role names. }
 \patrick{It might be possible to save space by making this into a figure spreading two columns.}
 \label{fig:calculus}
\end{figure}

A \emph{substitution} is a function replacing all occurrences of $x$ by a given term in a clause (or literal, or term) and a \emph{unifier} of two terms $t_1$ and $t_2$ is a substitution $\sigma$ s.t.\ $\sigma(t_1)=\sigma(t_2)$.
The rules R$A$, R$\quant$ and R$\forall$-1 rely on the most general unifier (mgu) $\sigma$ of $t_1$ and $t_2$, i.e., the identity if $t_1=t_2$ or a function mapping $x$ to $a\in\NI$ if one of $t_1$ and $t_2$ is $x$ and the other is~$a$.
If there is no mgu, a rule cannot be applied.
Inferences are also forbidden if the resulting clause contains more than a single literal of the form $\neg D(x)$ where $D\in\ND$. 

Let $\basedefs$ denote the set of definers introduced in $\Clauses$ in \text{\ref{step_one}}. %
Every subset $\subd$ of $\basedefs$ is mapped to a unique definer $D_\subd\in\ND$ s.t.\ $D_{\{D_1\}}=D_1$ for any $D_1\in\basedefs$.
 Intuitively, $D_\subd$ represents $D_1\sqcap\dots\sqcap D_n$ for $\subd = \{D_1,\dots,D_n\}$.
 Every time R$\quant$ is applied on definers $D_{\subd_1}$ and $D_{\subd_2}$, clauses
$\neg D_{\subd_1\cup \subd_2}(x)\vee D_{\subd_1}(x)$ and 
$\neg D_{\subd_1\cup \subd_2}(x)\vee D_{\subd_2}(x)$ are added, unless they already exist.
This can make further inferences on another concept or role name possible. 
For instance, we cannot apply R$A$ on the clauses $\neg D_1(x)\vee B(x)$ and $\neg D_3(x)\vee \neg B(x)$ in Ex.\ \ref{ex:normalise} since the resulting clause would contain two negative definers of the form $\neg D(x)$.
But if we first apply R$\quant$ on $\neg B(x)\vee\exists r.D_1(x)$ and $\neg D_2(x)\vee\forall r.D_3(x)$, this introduces the definer $D_{\{\subd_1\cup\subd_3\}}$, that we denote $D_{13}$ to lighten the notations, and the corresponding extra clauses.
After applying R$A$ on these clauses and the problematic ones, the obtained clauses $\neg D_{13}\vee B(x)$ and $\neg D_{13}\vee\neg B(x)$ can be resolved on $B$ using R$A$.%

To focus on inferences relevant to the observation,
we use an extension of the \emph{set-of-support} strategy \cite{plaisted1994search}.
The general idea of this strategy is to have a special set $\SupportClauses$ of \emph{supported 
clauses}. Inferences are then applied with the side condition that at least one premise is taken 
from $\SupportClauses$, while the newly inferred clauses are added to the set of supported clauses. 
In our context, we initialise $\Clauses$ as the normal clausal form of \Kmc; 
and~$\SupportClauses$ as the normal form of 
$\neg\Psi$. The set-of-support strategy thus makes sure that we only perform inferences that 
are connected to the observation. However, some modification to the standard set-of-support 
strategy are necessary due to the special role of definers. Intuitively, definers represent 
complex concepts under role restrictions, which have been introduced either as part of the 
normalisation step or by rule applications. When adding a new clause to $\SupportClauses$ that 
contains a definer, we have to make sure that the connections of this definer to other clauses are 
now considered in $\SupportClauses$ as well. 

To eliminate a name $S$, we perform all inferences on $S$ using the above strategy. Afterwards, we 
can remove from~$\SupportClauses$ all clauses that use $S$. It is now possible that previously 
filtered clauses containing $S$ get inferred again when eliminating a different name. To ensure 
termination, we thus have to keep track of previous inferences, for which we use the 
set~$\AllClauses$.

We now describe the algorithm used in \ref{step_two}. %
All inferences are performed under the condition that at least one premise is from 
\SupportClauses.
The set $\AllClauses$ is initialized as $\Clauses\cup\SupportClauses$. 
We then repeat the following steps for some name $S\in \sig{\SupportClauses}\setminus(\ND\cup\Sig)$ 
as long as there is such an $S$.
\begin{enumerate}[left=0pt,%
label=\textbf{F\arabic*}]
\item Perform all possible inferences on $S$, all possible R$A$ and R$\forall$-2 inferences on definers, and all R$\quant$ and R$\forall$-1 inferences that make inferences on $S$ possible.
  Add all inferred clauses not already in $\AllClauses$ to $\SupportClauses$ and~$\AllClauses$.
  \label{step:inferences}
\item Remove from~$\SupportClauses$ all clauses $\cls$ with $S\in\sig{\cls}$.
\label{step:remove-defining-clauses}
\item Repeatedly move to $\SupportClauses$ all clauses $\cls\in\Clauses$ in which a 
$D\in\sig{\SupportClauses}$ occurs until no more clauses are added.\label{step:move-clauses}
\end{enumerate}

The set $\AllClauses$ ensures the method terminates, as we may reintroduce a formerly eliminated name $S$ in~\ref{step:move-clauses}. Note that our normal form allows at most double-exponentially many distinct elements in $\AllClauses$. 
The final state of $\SupportClauses$ is denoted $\Sat{\Clauses,\SupportClauses,\Sigma}$. ($\mathsf{Sat}$ is short for saturation.)

\begin{example}\label{ex:saturate}
 Continuing with Ex.\ \ref{ex:normalise}, let us assume $\Sig_1=\{A_1,A_2,r\}$. We obtain 
 $\Sat{\Clauses,\SupportClauses,\Sig_1}=
 \{\neg A_1(a)\vee\exists\topRole.D_2(a)$, 
   $\neg D_2(x)\vee A_2(x)$,
   $\neg D_2(x)\vee \exists r.D_3(x)$,
   $\neg D_2(x)\vee \neg A_1(x)\vee\exists r.D_{13}(x)$,
   $\neg D_{13}(x)\}$. 
\end{example}

\rep{
To show that this procedure computes all relevant inferences, we use 
the refutational completeness of a modified version of our calculus. 
Given a set \Clauses of normalised clauses and a name $S\in\NC\cup\NR$, we denote by $\SatS{S}{\Clauses}$ the result of exhaustively applying the rules with the following 
additional side condition.
\begin{itemize}[label=($\ast$),leftmargin = 1.7em]
 \item If $S\in\sig{\cls}$, inferences on $\cls$ are either only applied 
on literals $L$ s.t.\ $S\in\sig{L}$, or they are applications of the R$A$-rule 
with a clause of the form $\neg D_1(x)\vee D_2(x)$. 
\end{itemize}
Intuitively, in $\SatS{S}{\Clauses}$, we first eliminate $S$ before performing any other inferences. 

\begin{lemma}\label{lem:calculus-refutational-completeness} Given a clause set \Clauses obtained by normalising a Boolean KB and some $S\in\NC\cup\NR$, $\Clauses$ is satisfiable if and only if $\SatS{S}{\Clauses}$ does not contain the empty clause.
\end{lemma}
\conf{
}
\rep{
\patrick{The following proof is almost identical to how the argument for my paper on forgetting with ALC-ABoxes would be - the only difference is treating literals of the form $\neg r(a,b)$, which did not exist in my paper, and the rules R$\exists$ and R$\topRole$, which are combined into one in the other calculus.}
\begin{proof}
Let $S$ and $\Clauses$ be as in the lemma, and assume $\SatS{S}{\Clauses}$ does not contain the empty clause.
 \Wlog, we assume $\Clauses$ contains at least one individual name (if not, we add a clause $A(a)$ s.t. $A\not\in\sig{N}$). 
 We introduce a fresh individual name $a_D$ for every definer $D$ occurring in $\SatClauses$. Let $\Inds$ be the set of all such individual names together with all individual names occurring in $\Clauses$.
 We define the \emph{grounding} of $\SatClauses$ as
 $
  \GroundClauses=\{\cls\rewriteto{x}{a} \mid a\in\Inds, C\in\SatClauses\},
$
and construct a model $\Imc$ of $\Clauses$ based on $\GroundClauses$.

Let $\prec$ be any total ordering on definers such that for $\mathbf{D}_1,\mathbf{D}_2\subseteq\ND*$, $D_{\mathbf{D}_1}\prec D_{\mathbf{D}_2}$ if $\mathbf{D_1}\subseteq\mathbf{D_2}$. 
%We extend $\prec$ from Lemma~\ref{lem:partialorder} to a total ordering on definer names.
 Let $\prec_S$ be a total ordering on literals s.t.
 $L_1\prec_s L_2$ satisfies the following, where each condition must be satisfied only if the preceding conditions are not satisfied.
 \begin{enumerate}[left=0pt]%[label=\textbf{O\arabic*}]
 	\item $L_1$ is of the forms $D(t)$ or $\neg D(t)$ and $L_2$ is not. %then $L_1\prec_S L_2$; otherwise
 	\item $L_1$ is of the form $D_1(t)$ or $\neg D_1(t)$ and $L_2$ is of the form $D_2(t')$ or $\neg D_2(t')$ and $D_1\prec D_2$. %then $L_1\prec_S L_2$; otherwise
 	\item $S\not\in\sig{L_1}$ and $S\in\sig{L_2}$. %then $L_1\prec_S L_2$; otherwise
 	\item $L_1$ is of the forms $A(t)$, $r(a,b)$ or $(\exists r.D)(t)$ and $L_2$ of the forms $\neg A'(t')$, $\neg r'(a',b')$ or $(\forall r'.D')(t')$. % then $L_1\prec_S L_2$; otherwise
 	\item $L_1$ is of the form $(\exists r.D_1)(t)$, $L_2$ is of the form $(\exists r.D_2)(t)$ and $D_1\prec D_2$. %, then $L_1\prec_S L_2$.
 \end{enumerate}
 \patrick{I note that you keep removing commas before the ``then'' and I keep adding them :D - I think they are optional in British English, but I find that in some cases, they improve readability.}
%Since each condition applies only if a previous condition has not already fixed the ordering of two literals, such an ordering always exists.
We extend $\prec_S$ to a 
%well-founded <-- Patrick: since we have no nested terms, and only finitely many clauses in complete, the point of well-foundedness is not so interesting in this case
%
total ordering on clauses using the standard multiset extension, i.e.\ $\cls_1\prec_S \cls_2$ if 
%and only if 
there exists some literal $L_2\in \cls_2$ s.t.\ for all literals $L_1\in \cls_1$ we have $L_1\prec_S L_2$.

%  Assume \Clauses is obtained by normalising a Boolean KB and $\SatClauses=\SatS{S}{\Clauses}$ for some arbitrary $S\in\NC\cup \NR$ does not contain the empty clause.
%Let $S$ and $\Clauses$ be as in the lemma, and assume $\SatS{S}{\Clauses}$ does not contain the empty clause.
% and $\SatClauses=\SatS{S}{\Clauses}$, where $\SatClauses$ does not contain the empty clause.
% \Wlog, we assume $\Clauses$ contains at least one individual name (if not, we add a clause $A(a)$ s.t. $A\not\in\sig{N}$).
% We introduce a fresh individual name $a_D$ for every definer $D$ occurring in $\SatClauses$, and let $\Inds$ be the set of all such individual names together with all individual names occurring in $\Clauses$.
% We define the \emph{grounding} of $\SatClauses$ as:
% \[
%  \GroundClauses=\{\cls\rewriteto{x}{a} \mid a\in\Inds, C\in\SatClauses\},
% \]
%and construct an interpretation $\Imc$ based on it and show that it is a model of $\Clauses$. 
\sophie{@Patrick: The above paragraph seems to have disappeared in the most recent version of the paper. Was this really intentional?}
\patrick{I just moved it to the beginning of the text :D }
 %In this interpretation, every individual name $a$ is represented by a domain element $d_a\in\Delta^\Imc$.
 %Moreover, for every definer $D$ occurring in $\Clauses$, we introduce a fresh individual name $a_D$ and its corresponding domain element $d_{a_D}\in\Delta^\Imc$, so that every domain element in our interpretation corresponds to some individual name, i.e.\  $\Delta^\Imc=\{d_a\,|\,a\text{ is an individual name}\}$.
% 
 $\Imc$ has exactly one domain element for every $a\in\Inds$, that is, $\Delta^\Imc=\{d_a\mid a\in\Inds\}$. Note that this implies that if $\Imc\models\GroundClauses$, also $\Imc\models\SatClauses$ and $\Imc\models\Clauses$.
 
 The interpretation function $\cdot^\Imc$ is inductively constructed as follows.
 $\Imc_0=\tup{\Delta^\Imc,\cdot^{\Imc_0}}$ is defined by setting: 
 \begin{enumerate}[left=0pt]
  \item for all $a\in\Inds$, $a^{\Imc_0}=d_a$;
  \item for all $D\in\ND$, $D^{\Imc_0}=\emptyset$ if $\neg D(x)\in\SatClauses$ and $D^{\Imc_0}=\{d_{a_D}\}$ otherwise;
  \item for all $U\in(\NC\cup\NR)\setminus\ND$, $U^{\Imc_0}=\emptyset$ .
 \end{enumerate}
 For $i>0$, while $\Imc_{i-1}\not\models\GroundClauses$, we define $\Imc_i=\tup{\Delta^\Imc,\cdot^{\Imc_i}}$ as an extension of $\Imc_{i-1}$.
 Let $\cls_m$ be the smallest clause in $\GroundClauses$ not entailed by $\Imc_{i-1}$ and $L$ be the maximal literal in $\cls_m$ according to the ordering $\prec$. 
 \begin{enumerate}[left=0pt,label=(\alph*)]
 		\item If $L=A(a)$, set $A^{\Imc_i}=A^{\Imc_{i-1}}\cup\{d_a\}$.
 		\item If $L=r(a,b)$, set $r^{\Imc_i}=r^{\Imc_{i-1}}\cup\{\reltup{d_a}{d_b}\}$.\label{step:model-construction-roles}
 		\item If $L=(\exists r.D)(a)$, set $r^{\Imc_i}=r^{\Imc_{i-1}}\cup\{\reltup{d_a}{d_{a_D}}\}$.
 		\item otherwise, set $\Imc_{i}=\Imc_{i-1}$.\label{step:model-construction-fail}
 \end{enumerate}
 For Case~(c), we note that $d_{a_D}\in D^{\Imc_0}$ unless $\neg D(x)\in\SatClauses$, in which case application of the R$\exists$-rule followed by the R$\topRole$ results in a clause $\cls_m'$ obtained from $\cls_m$ by removing $\exists r.D$. 
%  
% %  necessarily $d_{a_D}\in D^{\Imc_0}$, otherwise % the case where $L=(\exists r.D)(a)$ and $d_{a_D}\not\in D^{\Imc_{i+1}}=\emptyset$ is impossible since then
%  i) $\neg D(x)\in\SatClauses$, ii) the R$\exists$-rule applies, resulting in a clause in which $L=(\exists r.D)$ is replaced by $(\exists\topRole.D)(a)$, iii) then the R$\topRole$-rule applies on this clause, resulting in a clause $C_m'$ in which $(\exists\topRole.D)(a)$ is removed.
However then $\cls_m'\prec \cls_m$, and if $\Imc_{i-1}\not\models \cls_m$, also $\Imc_{i-1}\not\models \cls_m'$, a contradiction.
 % , contradicting that $C_m$ is the smallest clause not entailed by $\Imc_{i-1}$.
 
It can now be shown as in the proof for Theorem~2 in~\citet{KoopmannSchmidt2014} that for all $\cls\in\Sat{\Clauses}$ s.t.\ $\cls\prec \cls_m$, $\Imc_i\models \cls$, and that Case~\ref{step:model-construction-fail} never applies and $\Imc_i\models \cls_m$.
In fact, the only difference to the construction by \citet{KoopmannSchmidt2014} is that our normal form allows literals of the form $\neg r(a,b)$.
By contradiction, assume that case~\ref{step:model-construction-fail} applies from $\Imc_{i-1}$ to $\Imc_i$, where a clause $\cls_m$ with maximal literal $L=\neg r(a,b)$, that is, $\cls_m=\cls_m'\vee\neg r(a,b)$ is the smallest not entailed by $\Imc_i$. This implies $\reltup{a^\Imc}{b^\Imc}\in r^{\Imc_i}$. Both $a$ and $b$ must be individual names occurring in $\Clauses$, which means $b^\Imc$ is not an introduced individual name $a_D$ for some definer $D$. Consequently, $\reltup{a^\Imc}{b^\Imc}\in r^{\Imc_i}$ because case~\ref{step:model-construction-roles} applied for some clause $\cls=\cls'\vee r(a,b)$ in which $r(a,b)$ is maximal, on some interpretation $\Imc_j$, where $j<i$. The latter implies that $\cls\prec \cls'$, but due to the R$r$ rule, we also have $\cls'\vee\cls_m'\in\SatClauses$.
Thus $\Imc_i\not\models\cls'\vee\cls_m'$ and $\cls'\vee\cls_m'\prec\cls_m$, a contradiction to the minimality of $\cls_m$. %, which means $\cls_m$ was not the smallest clause according to $\prec$ not entailed by $\Imc_i$. A contradiction.

Since $\GroundClauses$ is finite and a new clause in $\GroundClauses$ becomes entailed at each step and all smaller clauses in $\GroundClauses$ stay entailed, there exists an $i>0$ such that $\Imc_i\models\GroundClauses$.
We define $\Imc = \Imc_i$, and note that also $\Imc\models\SatClauses$ and $\Imc\models\Clauses$.
Consequently, $\Clauses$ is satisfiable. 
\end{proof}

}

We furthermore need to prove properties on introduced definers.
Let $\Sat{\Clauses}$ be the result of exhaustively applying the rules of the calculus without the condition ($\ast$).
The following lemma can be shown by induction on the sets
$\mathbf{D}\subseteq\basedefs$ introduced by the R$\quant$ rule.
\begin{lemma}\label{l:definer-union}
Let $\Clauses$ be a normalised set of clauses with definers $\basedefs$. 
For every $\mathbf{D}\subseteq\basedefs$\, for 
which a definer $D_\mathbf{D}$ is introduced in
$\Sat{\Clauses}$ and every $\mathbf{D}'\subset\mathbf{D}$, $\neg 
D_{\mathbf{D}}(x)\vee D_{\mathbf{D}'}(x)\in\Sat{\Clauses}$ and there exist 
$\varphi_1\vee\quant_1.D_{\mathbf{D}'}$, 
$\varphi_1\vee\varphi_2\vee\quant_2.D_{\mathbf{D}}\in\Sat{\Clauses}$, where 
$\quant_1=\exists$ implies $\quant_2 =\exists$.
\end{lemma}

\begin{lemma}\label{l:definers}
 Let $\Clauses_1$ and $\Clauses_2$ be sets of normalised clauses with
$\Clauses_1\subseteq\Clauses_2$, and $D$, $D_1$, $D_2$ be definers with $D_1$, 
$D_2\in\sig{\Sat{\Clauses_1}}$ and $\neg D(x)\vee D_1(x)$, $\neg D(x)\vee 
D_2(x)\in\Sat{\Clauses_2}$. Then, there exists a definer $D'$ with $\neg 
D'(x)\vee D_1(x)$, $\neg D'(x)\vee D_2(x)\in\Sat{\Clauses_1}$, and either 
$D'=D$ or $\neg D(x)\vee D'(x)\in\Sat{\Clauses_2}$.
\end{lemma}
\conf{
  \begin{proofsk}
    Let $D_1 = D_{\subd_1}$ and $D_2 = D_{\subd_2}$.
    By Lemma~\ref{l:definer-union}, there exists a definer $D_{\subd_1\cup\subd_2}$, denoted 
    $D_{12}$, s.t.\ $\neg D(x)\vee D_{12}(x)\in\sig{\Clauses_2}$.
    That $D_{12}$ is introduced in $\Sat{\Clauses_1}$ can be shown by contradiction.
    We assume that it is the first such definer introduced in $\Sat{\Clauses_2}$ that is not 
    introduced in $\Sat{\Clauses_1}$ and we show that necessarily another such definer must have 
been introduced before.
  \end{proofsk}
}
\rep{
\begin{proof}
Let $D_1 = D_{\subd_1}$ and $D_2 = D_{\subd_2}$.
By Lemma~\ref{l:definer-union}, there exist a definer $D_{\subd_1\cup\subd_2}$, that we denote $D_{12}$ to lighten the notations, %
s.t.\ $\neg D(x)\vee D_{12}(x)\in\sig{\Clauses_2}$. We show that 
$D_{12}$ is introduced in $\Sat{\Clauses_1}$.
For a proof by contradiction, 
assume that in a sequence of inferences, $D_{12}$ is  the first such 
definer introduced in $\Sat{\Clauses_2}$ that is not introduced in 
$\Sat{\Clauses_1}$. 
Let $\varphi_1\vee\quant_1 r.D_1(t_1)$ and $\varphi_2\vee\quant_2 r.D_2(t_2)$ 
be the occurrences of $D_1$ and $D_2$ in $\Sat{\Clauses_1}$. Since $D_{12}$ is 
not introduced in $\Sat{\Clauses_1}$, R$\quant$ does not apply on these 
clauses, which is only the case if $\varphi_1$ and $\varphi_2$ respectively 
contain literals $\neg D_1'(x)$ and $\neg D_2'(x)$ where $D_1'\neq 
D_2'$. Since $\Sat{\Clauses_2}$ contains $D_{12}$, it must contain clauses 
$\varphi_1'\vee\quant_1 r.D_1(t_1)$ and $\varphi_2'\vee\quant_2 r.D_2(t_2)$ 
which have been inferred from $\varphi_1\vee\quant_1 r.D_1(t_1)$ and 
$\varphi_2\vee\quant_2 r.D_2(t_2)$ through a sequence of inferences, and where 
$\varphi_1'$ and $\varphi_2'$ do not contain different negative definers. 
Consequently, $\varphi_1'$ and $\varphi_2'$ must contain a literal $D'$ s.t.\ 
$\neg D'(x)\vee D_1'(x)$, $\neg D'(x)\vee D_2'(x)\in\Sat{\Clauses_2}$ but $D'$ 
is not introduced in $\Sat{\Clauses_1}$, nor is a corresponding definer $D''$ with $\neg D'(x)\vee D''(x)\in\Sat{\Clauses_2}$. Furthermore, $D'$ must have been 
introduced before $D_{12}$ could have been introduced.
This contradicts the assumption that $D_{12}$ is the first such definer introduced in $\Sat{\Clauses_2}$ that is not introduced in $\Sat{\Clauses_1}$. 
\end{proof}
}

We are now able to prove that the set $\Sat{\Clauses,\SupportClauses,\Sigma}$ 
computed in \ref{step_two} %
preserves all relevant consequences.
}

\conf{
The following central theorem is proved in the extended version of the paper.
}
\begin{restatable}{theorem}{LemCompleteness}\label{lem:completeness}
 Let $\Clauses$ and $\SupportClauses$ be normalised clause sets and $\Sig$ be a signature.
 Let $\mathbf{M}_\Sig=\Sat{\Clauses,\SupportClauses,\Sig}$.
 Then, 
$\Clauses\cup\SupportClauses\models\Kmc$ if and only if $\Clauses\cup \mathbf{M}_\Sig\models\Kmc$, for every Boolean $\ALC$ KB~$\Kmc$ s.t.\ $\sig{\Kmc}\subseteq\Sig$.
\end{restatable}
\rep{
\conf{
 \begin{proofsk} 
   In each iteration of the loop in the algorithm for computing $\Sat{\Clauses,\SupportClauses,\Sig}$, a name $S\in\sig{\SupportClauses}\setminus(\ND\cup\Sig)$ is processed. We show that each such step preserves entailments modulo $S$. Specifically, let $\SupportClauses^0$ be the set of supported clauses before $S$ is processed, $\SupportClauses^{\textbf{F1}}$ the outcome of~\ref{step:inferences}, and $\SupportClauses^{\textbf{F2}}$ the result of~\ref{step:remove-defining-clauses}. 
   Let $\Clauses^0$ be the original set $\Clauses$ of unsupported clauses.
   Let $\Kmc$ be a Boolean \ALC KB s.t.\ $S\not\in\sig{\Kmc}$. We show that $\Clauses^0\cup\SupportClauses^0\models\Kmc$ if and only if $\Clauses^0\cup\SupportClauses^{\textbf{F2}}\models\Kmc$. The theorem follows by induction over the inference steps performed.
 \end{proofsk}
}
}
\rep{
\begin{proof}

In each iteration of the loop in the algorithm for computing $\Sat{\Clauses,\SupportClauses,\Sig}$, a name $S\in\sig{\SupportClauses}\setminus(\ND\cup\Sig)$ is processed. We show that each such step preserves entailments modulo $S$. Specifically, Let $\SupportClauses^0$ be the set of supported clauses before $S$ is processed, $\SupportClauses^{\textbf{F1}}$ the outcome of~\ref{step:inferences}, and $\SupportClauses^{\textbf{F2}}$ the result of~\ref{step:remove-defining-clauses}. 
Let $\Clauses^0$ be the original set $\Clauses$ of unsupported clauses.
Let $\Kmc$ be a Boolean \ALC KB s.t.\ $S\not\in\sig{\Kmc}$. We show that $\Clauses^0\cup\SupportClauses^0\models\Kmc$ if and only if $\Clauses^0\cup\SupportClauses^{\textbf{F2}}\models\Kmc$. The theorem follows by induction on the inference steps.

\newcommand{\AddClauses}{\ensuremath{\Clauses_{\neg\Kmc}}\xspace}

\newcommand{\Ma}{\ensuremath{\textbf{M}_1}\xspace}
\newcommand{\Mb}{\ensuremath{\textbf{M}_2}\xspace}

For $a\in\{1,3\}$, $\Clauses^0\cup\SupportClauses^a\models\Kmc$ if and only if $\Clauses^0\cup\SupportClauses^a\cup\{\neg\Kmc\}$ is not satisfiable. We represent $\neg\Kmc$ as set $\AddClauses$ of clauses as described under \ref{step_one} %
of the abduction procedure. 
We then need to show that, for $\Ma=\Clauses^0\cup\SupportClauses^0\cup\AddClauses$ and $\Mb=\Clauses^0\cup\SupportClauses^{\textbf{F2}}\cup\AddClauses$, 
$\Ma$ is unsatisfiable if and only if so is $\Mb$.
By Lemma~\ref{lem:calculus-refutational-completeness}, this can be reduced to
showing $\SatS{S}{\Ma}$ contains the empty clause if and only if so does $\SatS{S}{\Mb}$.
If $\SatS{S}{\Mb}$ contains the empty clause, then so does $\SatS{S}{\Ma}$,
since $\Mb$ contains only clauses that can be inferred using the calculus on
$\Ma$.
In the other direction, we need to show that inferences on clauses in $\SupportClauses^{\textbf{F1}}\setminus\SupportClauses^{\textbf{F2}}\subseteq\Ma$ can be recovered.

\textsc{Claim.} For every clause $\cls\in\SatS{S}{\Ma}$ that is inferred 
using an inference on $S$, there exists a clause 
$\cls'\in\SatS{S}{\Clauses^0}\cup\SupportClauses^{\textbf{F1}}$ s.t.: 
\begin{enumerate}[left=0pt,label=C\arabic*.]
\item $\cls=\cls'$ or 
\item $\cls=\neg D_1(x)\vee \cls_r$, $\cls'=\neg D_2(x)\vee \cls_r$ and $\neg D_1(x)\vee D_2(x)\in\SatS{S}{\Mb}$.
\end{enumerate}

We prove the claim by induction on the inferences. 
Let $\cls$ be a clause in $\SatS{S}{\Ma}$ inferred by an inference on 
$S$, that is, from a clause $\cls_1$ with $S\in\sig{\cls_1}$ and possibly 
another clause $\cls_2$ with $S\in\sig{\cls_2}$. By Condition~($\ast$), $\cls_1$ 
either:
\begin{enumerate}[left=0pt,label=($\ast$\arabic*)]
\item occurs in $\Clauses_0\cup\SupportClauses^0$, or\label{first-star}
\item is the conclusion of an inference on $S$ or\label{second-star}
\item of an inference on $\neg D_1(x)\vee D_2(x)\in\SatS{S}{\Ma}$.\label{last-star}
\end{enumerate}
If \ref{last-star} applies, possibly due several 
applications of R$A$, we must have $\cls_1=\neg D_1(x)\vee\cls_1'$, $\neg 
D_1(x)\vee D_3(x)\in\SatS{S}{\Ma}$ and $\neg D_3(x)\vee 
\cls_1'\in\SatS{S}{\Ma}$, where \ref{first-star} and \ref{second-star} apply to $\neg 
D_3(x)\vee\cls_1'$. If \ref{first-star} applies to both premises, the claim holds 
directly. If \ref{first-star} or \ref{second-star} apply and $\cls_1$, 
$\cls_2\in\Sat{\Clauses^0}\cup\SupportClauses^{\textbf{F1}}$, then 
$\cls\in\Sat{\Clauses^0}\cup\SupportClauses^{\textbf{F1}}$ by construction of 
$\SupportClauses^{\textbf{F1}}$. If $\cls\not\in\Sat{\Clauses^0}\cup\SupportClauses^{\textbf{F1}}$, this 
can only be because $\cls$ contains a definer not in 
$\Sat{\Clauses^0}\cup\SupportClauses^{\textbf{F1}}$. Specifically, we must have $\cls=\neg 
D(x)\vee \cls_r$.
Assume Case~\ref{second-star} applies and the claim holds for $\cls_1$ and $\cls_2$ (the argument for Case~\ref{last-star} is the same, but uses $\cls_1'$ and/or $\cls_2'$). 
$\SatS{S}{\Mb}$ then contains $\cls_1=\neg D(x)\vee\cls_{r1}$, 
$\cls_2=\neg D(x)\vee\cls_{r2}$, $\neg D(x)\vee D_1(x)$ and $\neg D(x)\vee 
D_2(x)$, and $\SatS{S}{\Clauses^0}\cup\SupportClauses^{\textbf{F1}}$ contains $\neg 
D_1(x)\vee\cls_{r1}$ and $\neg 
D_2(x)\vee\cls_{r2}$. By 
Lemma~\ref{l:definers}, there then exists a definer $D'$ s.t.\ $\neg 
D'(x)\vee\cls_{r1}$, $\neg 
D'(x)\vee\cls_{r2}\in\SatS{S}{\Clauses^0}\cup\SupportClauses^{\textbf{F1}}$, and either 
$D'=D$ or $\neg D(x)\vee D'(x)\in\SatS{S}{\Mb}$ (inferred on the clauses after inferences on $S$ have been applied). Any inference performed on 
$\cls_1$ and $\cls_2$ also applies on those clauses, 
resulting in the clause $\neg D'(x)\vee\cls_r\in\SatS{S}{\Clauses^0}\cup\SupportClauses^{\textbf{F1}}$ satisfying the claim. 

By the claim, all relevant 
inferences on $S$ are performed when computing 
$\SatS{S}{\Clauses^0}\cup\SupportClauses^{\textbf{F1}}$. $\SupportClauses^{\textbf{F2}}$ is obtained 
from $\SupportClauses^{\textbf{F1}}$ by removing only clauses that contain $S$. 
By Condition~($\ast$), for computing $\SatS{S}{\Ma}$, inferences on $S$ and positive 
definers are applied before any other inferences. We thus obtain that for every 
clause $\cls\in\SatS{S}{\Ma}$, there 
exists a clause $\cls'\in\SatS{S}{\Mb}$ s.t.\ 
1) $\cls=\cls'$ or 2) $\cls=\neg D_1(x)\vee \cls_r$, $\cls'=\neg D_2(x)\vee 
\cls_r$ and $\neg D_1(x)\vee D_2(x)\in\SatS{S}{\Mb}$. It follows directly that 
if $\SatS{S}{\Ma}$ contains the empty clause, so does $\SatS{S}{\Mb}$.
\end{proof}
}

\subsection{Step 3: Denormalisation}

\newcommand{\StepThreeClauses}{\mathbf{M}_\Sig}

Next, we turn $\StepThreeClauses=\Sat{\Clauses,\SupportClauses,\Sig}$ into a definer-free Boolean \ALCOImuTopRole KB that preserves all entailments in \ALC modulo definer names. Each definer represents a concept that occurs under a role restriction. To eliminate the definers, we compute CIs of the form $D\sqsubseteq C$ (the definition of $D$), so that occurrences of $D$ can be replaced by $C$. While such a step is also performed in the resolution-based uniform interpolation method in~\cite{KoopmannSchmidt2014}, we have to do more in our current setting. For uniform interpolation of \ALC KBs, it is sufficient to look at clauses of the form $\neg D(x)\vee \cls$ to build the definition $D\sqsubseteq C$ of $D$. In our setting, special care has to be taken also of clauses of the form $\neg D(a)\vee \cls$: first, the definition of $D$ has to refer to $a$, for which we use nominals. Second, if we have a clause $\cls'\vee \forall r.D(t)$, we also have to consider substitutions of $\neg D(a)$ with $(\forall r^-.C)(a)$, where $C$ is a concept that corresponds to the clause $\cls'$.

Since we applied the R$A$-rule exhaustively on positive definer occurrences, clauses that contain literals of the form~$D(t)$ are not needed anymore at this stage and are removed before further operations are performed. 
We first introduce a concept-representation of clauses. For every definer $D$ occurring in $\StepThreeClauses$, we introduce a fresh definer $\overline{D}$ representing~$\neg D$. Given a concept $C$, we denote by $C^-$ the result of replacing every concept of the form $\neg D$ by $\overline{D}$.
Given a clause $\cls=L_1\vee\ldots\vee L_n$, we define the concept $\cls^c=L_1^c\sqcup\ldots\sqcup L_n^c$, where %
$L^c$ is defined as:
\begin{enumerate}[left=0pt,%
label=\textbf{C\arabic*}]
 \item $C^-$ if $L=C(x)$,
 \item $\exists\topRole.(\{a\}\sqcap C^-)$ if $L=C(a)$,
 \item $\exists\topRole.(\{a\}\sqcap\exists r.\{b\})$ if $L=r(a,b)$,
 \item $\exists\topRole.(\{a\}\sqcap\forall r.\neg\{b\})$ if $L=\neg r(a,b)$.
\end{enumerate}
Every clause contains either only variables or only individual names as terms. In the former case, $\cls^c$ is the concept described by the clause. In the latter case, $\cls^c$ is such that for every interpretation $\Imc$ and every $d\in\Delta^\Imc$, $\Imc\models \cls$ if and only if $d\in(\cls^c)^\Imc$, that is, either $(\cls^c)^\Imc=\emptyset$ or $(\cls^c)^\Imc=\Delta^\Imc$ depending on whether $\cls$ is entailed.
As said, all universal roles introduced here are eliminated in \ref{step_four}. %

We then build a set of CIs giving meaning to the definers: 
\begin{enumerate}[left=0pt,%
label=\textbf{D\arabic*}]
 \item for every $\cls\vee\forall r.D(x)$, we add $\overline{D}\sqsubseteq\forall r^-.\cls^c$,
 \label{def-CIs-first}
 \item for every $\cls\vee\forall r.D(a)$, we add $\overline{D}\sqsubseteq\forall r^-.(\neg\{a\}\sqcup \cls^c)$,
 \label{def-CIs-second}
 \item every $\neg D(x)\vee \cls$ gets replaced by $D\sqsubseteq \cls^c$,\label{def-CIs-third}
 \item for every $\neg D(a)\vee \cls$, we add $D\sqsubseteq\neg\{a\}\sqcup \cls^c$, 
 \item in every clause, we replace literals $\neg D(a)$ by $\overline{D}(a)$, and
 \item every remaining $\cls$ that is not a disjunction of (negated) ABox assertions is replaced by $\top\sqsubseteq \cls^c$.\label{def-CIs-last}
\end{enumerate}
The transformation makes sure that the meaning of each definer $D$ is captured by CIs $D\sqsubseteq C$, and that all remaining negative occurrences of $D$ now refer to the definer $\overline{D}$.
Addition of axioms happens first, then replacement of clauses.

We apply \ref{def-CIs-first}--\ref{def-CIs-second} on $\StepThreeClauses$ and denote the result by $\Kmc_0$.

\begin{lemma}\label{lem:new-neg-definers}
 Every model of $\Kmc_0$ can be transformed into a model of 
 $
 \Kmc_0\cup\{\overline{D}\equiv\neg D\mid D\in\sig{\StepThreeClauses}\}
 $
 by only changing the interpretation of definers.
\end{lemma}
\conf{
  \begin{proofsk}
    We show how to construct a model of $\Kmc_0\cup\{\overline{D}\equiv\neg D\mid D\in\sig{\StepThreeClauses}\}$ iteratively by working on one~$D$ at a time.
    Let $\Imc$ be a model of $\Kmc_0$ s.t.\ for some definer $D\in\sig{\StepThreeClauses}$, $\Imc\not\models \overline{D}\equiv\neg D$.
    The normalisation and the calculus make sure that $D$ occurs either only under existential restrictions or only under universal restrictions.
    In the former case, we note that no CIs are introduced for~$\overline{D}$, and we can obtain a model $\Imc'$ with $\Imc'\models\overline{D}\equiv\neg D$ by setting $\overline{D}^{\Imc'}=(\neg D)^\Imc$ and keeping everything else the same.
    Otherwise, $D$ occurs only under universal restrictions, and we transform $\Imc$ into a model $\Imc'$ of $\overline{D}\equiv\neg D$ by setting \(D^{\Imc'}=D^\Imc\setminus\overline{D}^\Imc\) and \( \overline{D}^{\Imc'}=\Delta^\Imc\setminus D^{\Imc'} \).
    Then we prove that $\Imc'$ is still a model of $\Kmc_0$.
  \end{proofsk}
}
\rep{
\begin{proof}
 Let $\Imc$ be a model of $\Kmc_0$ s.t.\ for some definer $D\in\sig{\Clauses}$, $\Imc\not\models \overline{D}\equiv\neg D$. The normalisation and the calculus make sure that $D$ either occurs only under existential restrictions or only under universal restrictions. In the former case, we note that no CIs are introduced for $\overline{D}$, and we can obtain a model $\Imc'$ with $\Imc'\models\overline{D}\equiv\neg D$ by setting $\overline{D}^{\Imc'}=(\neg D)^\Imc$ and keeping everything else the same. Otherwise, $D$ occurs only under universal restrictions, 
 and we transform $\Imc$ into a model $\Imc'$ of $\overline{D}\equiv\neg D$ by setting:
\(
 D^{\Imc'}=D^\Imc\setminus\overline{D}^\Imc\) and %
\( \overline{D}^{\Imc'}=\Delta^\Imc\setminus D^{\Imc'}
\).

Clearly, $\Imc'\models \overline{D}\equiv\neg D$. We show that also
$\Imc'\models\Kmc_0$, for which we need to consider occurrences of $D$ and
$\overline{D}$. For $D$, since $D^{\Imc'}\subseteq D^\Imc$, we only need to
consider positive occurrences in $\Kmc_0$.
Since $D$ does not occur under existential role restrictions,
and clauses with literals of the form $D(t)$ are eliminated by the saturation procedure,
the only occurrences of $D$ we need to consider are of the form $\cls\vee\forall r.D(t)$.

Recall that every clause either contains only variables or only individual names as terms.
Assume $L_1(x)\vee\ldots L_n(x)\vee\forall r.D(x)\in \Kmc_0$, 
and assume there exists some $\reltup{d}{e}\in r^\Imc$ s.t.\ $e\in(D^{\Imc}\cap\overline{D}^\Imc)$. Since $e\in\overline{D}^\Imc$ and $\Imc\models \overline{D}\sqsubseteq\forall r^-.(L_1^c\sqcup\ldots\sqcup L_n^c)$, we obtain $d\in(L_1^c\sqcup\ldots\sqcup L_1^c)^\Imc$,
which means $d\in L_i$ for some $1\leq i\leq n$. Consequently, $e$ does not need to be in $D^\Imc$ for the clause to be satisfied, and we obtain also $\Imc'\models L_1(x)\vee\ldots L_n(x)\vee\forall r.D(x)$. Consider $\cls\vee\forall r.D(a)\in\Kmc_0$, $\Imc\not\models \cls$ and $\reltup{a^\Imc}{d}\in r^{\Imc}$. We then have 
$\overline{D}\sqsubseteq\forall r^-.(\neg\{a\}\sqcup \cls^c)\in\Kmc_0$. From $\Imc\not\models \cls$, it follows that $(\cls^c)^\Imc=\emptyset$,
and consequently that $\Imc\models\overline{D}\sqsubseteq\forall r^-.\neg\{a\}$. Therefore, $a^\Imc$ cannot have an $r$-successor that satisfies $\overline{D}$, so that $\Imc\models \cls\vee\forall r.D(a)$ implies $\Imc'\models \cls\vee\forall r.D(a)$. 

It remains to consider $\overline{D}$. We note that $\overline{D}^{\Imc'}\supseteq\overline{D}^\Imc$, and thus we have to consider the negative occurrences of $\overline{D}$ in $\Kmc_0$, which are all of the form $\overline{D}\sqsubseteq\forall r^-.C$, and domain elements $d\in\overline{D}^{\Imc'}\setminus\overline{D}^\Imc$. We first note that for those domain elements $d$, $d\not\in D^\Imc$, since $d\not\in\overline{D}^\Imc$ and $d\in D^\Imc$ would also imply $d\in D^{\Imc'}$ and $d\not\in\overline{D}^{\Imc'}$. 
We show that for every $e$ s.t. $\reltup{e}{d}\in r^\Imc$, $d\in C$. We consider the different possibilities for $\overline{D}\sqsubseteq\forall r^-.C\in\Kmc_0$, and show that in each case, $e\in C^{\Imc'}$. 1) $\overline{D}\sqsubseteq\forall r^-.(L_1^c\sqcup\ldots\sqcup L_n^c)\in\Kmc_0$ generated from $L_1(x)\vee\ldots\vee L_n(x)\vee\forall r.D(x)\in\Kmc_0$. Since $d\not\in D^\Imc$, $e\not\in(\forall r.D)^\Imc$, and therefore $e\in(L_1\sqcup\ldots\sqcup L_n)^{\Imc'}$. 2) $\overline{D}\sqsubseteq\forall r^-.(\neg\{a\}\sqcup \cls^c)\in\Kmc_0$ generated from $\cls\vee\forall r.D(a)$.
If $e\neq a^\Imc$, then $e\in(\neg\{a\}\sqcup \cls^c))^{\Imc'}$.
Otherwise, $e=a^\Imc$, and since $d\not\in D^\Imc$ it follows that $\Imc\not\models\forall r.D(a)$, and consequently $\Imc\models \cls$, which in turn implies $e(\cls^c)^\Imc=\Delta^\Imc$ and $e\in(\neg\{a\}\sqcup \cls^c))^{\Imc'}$
\sophie{Explain previous statement in more details (reason why $\cls^c$ is either everything or nothing)}
\patrick{Better? Note that I mentioned the all or nothing aspect of $\cls^c$ before the lemma.}
It follows that %
$\Imc'\models \overline{D}\sqsubseteq \forall r^-.C$ for every $\overline{D}\sqsubseteq\forall r^-.C\in\Kmc_0$. 

We obtain that $\Imc'\models\Kmc_0$. We can now proceed in that manner iteratively for every definer $D\in\sig{\StepThreeClauses}$ to obtain a model $\Imc^*$ of $\Kmc_0\cup\{\overline{D}\equiv\neg D\mid D\in\sig{\StepThreeClauses}\}$.
\end{proof}
}

Let $\Kmc_1$ be the result of applying~\ref{def-CIs-first}--\ref{def-CIs-last}.

\begin{lemma}\label{lem:definer-elimination-preparation}
  Every model of $\Kmc_1$ can be turned into a model of $\StepThreeClauses$ by only changing the interpretation of definers, and every model of $\StepThreeClauses$ can be extended to a model of $\Kmc_0$ by setting $\overline{D}^\Imc=\neg D^\Imc$ for all definers $D$ occurring in $\StepThreeClauses$.
\end{lemma}
\begin{proof}
 Looking at the axioms added in~\ref{def-CIs-third}--\ref{def-CIs-last}, one obtains that $\Kmc_1\models\Kmc_0$ 
 and $\Kmc_0\cup\{\overline{D}\equiv\neg D\mid D\in\sig{\StepThreeClauses}\}\models\Kmc_1$. 
 By Lemma~\ref{lem:new-neg-definers}, we can thus transform any model $\Imc$ of $\Kmc_1$ into a model $\Imc'$ of $\Kmc_0$ s.t.\ $\Imc'\models\overline{D}\equiv\neg D$ for every $D\in\sig{\StepThreeClauses}$, by only changing the interpretation of the definers in $\Imc$. $\Imc'$ is also a model of $\Kmc_1$. 
 Inspection of $\Kmc_1$ shows that every clause $\cls\in\StepThreeClauses$ has an axiom $\alpha\in\Kmc_0$ s.t.\ $\Imc'\models \cls$ if and only if $\Imc'\models\alpha$. It follows that $\Imc'$ is a model of $\StepThreeClauses$.

 Let $\Imc$ be a model of $\StepThreeClauses$ and let $\Imc'$ coincide with $\Imc$, except on $(\overline{D})^{\Imc'}=\neg D^\Imc$ for every definer $D$ occurring in $\StepThreeClauses$.
 Again, an inspection of $\Kmc_1$ reveals that, for every axiom $\alpha\in\Kmc_0$, there is some clause $\cls\in\StepThreeClauses$ s.t.\ $\Imc'\models \cls$ if and only if $\Imc'\models\alpha$. It follows that $\Imc'$ is a model of $\Kmc_1$.
 \end{proof}
 
 As a consequence of Lemma~\ref{lem:definer-elimination-preparation}, $\Kmc_1$ preserves all relevant entailments of $\StepThreeClauses$ not using definers, while now every negative occurrence of a definer is in a CI of the form $D\sqsubseteq C$. This allows us now to use the same technique as in~\cite{KoopmannSchmidt2014} to eliminate definers.
For each definer $D$ occurring in $\Kmc_1$, including the introduced definers~$\overline{D}$, we replace all axioms
$D\sqsubseteq C_1$, $\ldots$, $D\sqsubseteq C_n$ by a single axiom $D\sqsubseteq
C_1\sqcap\ldots\sqcap C_n$. If there is no such axiom, we add
$D\sqsubseteq\top$. We then eliminate each definer $D$ as follows:
\begin{enumerate}[left=0pt,%
label=\textbf{E\arabic*}]
 \item Each CI $D\sqsubseteq C[D]$ ($D$ occurs on both sides of the CI), is replaced by $D\sqsubseteq\nu X.C\rewriteto{D}{X}$,
 \label{def-elim-fixpoint}
 \item For every CI $\alpha=D\sqsubseteq C$, where $D\not\in \sig{C}$, remove~$\alpha$ and replace $D$ everywhere by $C$.\label{def-elim-replace}
\end{enumerate}
The result is then denoted by $\Kmc_2$. The next theorem follows from Lemma~1 and~2 
together with Ackermann's Lemma~\cite{ACKERMANN} and Generalised Ackermann's 
Lemma~\cite{GEN-ACKERMANN}, which concern the corresponding transformation in second-order logics. 
\begin{restatable}{theorem}{LemDeclausification}\label{lem:declausification}
 For every definer-free Boolean KB $\Kmc$, $\Kmc_2\models\Kmc$ if and only if $\StepThreeClauses\models\Kmc$.
\end{restatable}

\subsection{Step 4: Negate Result}

The CIs in the Boolean KB after \ref{step_three} %
are only needed to eliminate definers and can be discarded.
Thus, we obtain a conjunction of disjunctions of (possibly negated) assertions, which can be directly negated to obtain a disjunction of conjunctions of assertions. From this, we eliminate occurrences of universal roles and make sure that only least fixpoint expressions are used. 
Universal roles initially represent the negations of CIs.
Further universal roles are introduced in \textbf{Steps~2} and~\textbf{3}. 
In each case, universal roles are only used in existential role restrictions. Eliminating definers only introduces greatest fixpoints. 
Thus, by pushing negations inside, we can make sure that the hypothesis contains only least fixpoints, and universal roles occur only in universal role restrictions. 
To obtain the final hypothesis, we pull out universal roles using the following equivalences together with standard DNF transformations:
\begin{align*}
 \exists r.(C_1\sqcap\forall\topRole.C_2) &\Leftrightarrow \exists r.C_1\sqcap \forall\topRole.C_2 \\
 \forall r.(C_1\sqcup\forall\topRole.C_2) &\Leftrightarrow \forall r.C_1\sqcup\forall\topRole.C_2\\
 (\forall\topRole.C)(a) &\Leftrightarrow \top\sqsubseteq C
\end{align*}
Finally, to ensure that the hypothesis returned is consistent, the check $\Kmc\wedge \hyp \not\models \perp$ is performed using an external reasoner. This eliminates false hypotheses in cases for which there is no suitable hypothesis to explain $\Psi$ in the given signature of abducibles $\Sig$.
 
A direct consequence of Theorems~\ref{lem:completeness} and~\ref{lem:declausification} is:
\begin{theorem}
Let $\tup{\Kmc,\Psi,\Sig}$ be an abduction problem.
Applying \textbf{Steps 1--4} 
to this problem produces a disjunction of \ALCOImu KBs that is a solution of the abduction problem.
\end{theorem}

\renate{Give steps?}
\begin{example}\label{ex:denormalise}
 For the clauses $\StepThreeClauses$ %
from Ex.~\ref{ex:saturate}, we obtain 
 the Boolean KB $\neg A_1(a)\vee\exists\topRole.(A_1\sqcap\forall r.\top\sqcap(\neg A_2\sqcap\exists r.\bot))(a)$, which can be simplified to $\neg A_1(a)\vee\exists\topRole.(A_1\sqcap \neg A_2)(a)$.
 The final hypothesis is equivalent to $A_1(a)\wedge (A_1\sqsubseteq A_2)$.
\end{example}

\section{Spaces of Independent Explanations}\label{sec:redundancy}
\label{hypothesis-filtering}
The abduction problem in Def.~\ref{KBAbductionProblem} searches for semantically minimal 
hypotheses. 
However, without extra constraints, the
result from the method in the last section may include redundant disjuncts.
Consider the following example%
\begin{align*}
&\KB = \{F \sqsubseteq B, D \sqcap E \sqsubseteq \perp, C(a), E(a)\} \\
& \psi = (\exists r.A) \sqcap C \sqsubseteq (\exists r.B) \sqcup D \\
& \Sig = \{A, B, C, D, E, F\}
\end{align*}
and the following three hypotheses $\hyp_1 = A \sqsubseteq B \lor C \sqsubseteq D$, 
$\hyp_2 = A \sqsubseteq B$ and $\hyp_3 = A \sqsubseteq F$. All of these 
satisfy~\ref{en:consistent}--~\ref{en:signature} in 
Def.~\ref{KBAbductionProblem}. $\Kmc\wedge\hyp_3$ entails both $\hyp_1$ and 
$\hyp_2$, and 
thus $\hyp_3$ does not satisfy ~\ref{en:semantic-minimality}.
However, it is also the case that $\KB \wedge \hyp_1 \equiv \KB \wedge \hyp_2$. In fact, both 
$\hyp_1$ and $\hyp_2$ are semantically minimal.
However, $\hyp_2$ is shorter, and would thus be preferred as an answer, while $\hyp_1$ contains a 
disjunct that is inconsistent with $\Kmc$.

To account for these inter-disjunct redundancies, we extend Def.~\ref{KBAbductionProblem}, following \cite{Del-PintoSchmidt2019}.  For this, we first need to make sure that disjunctions are pulled out where possible: a disjunction of KBs is in \emph{disjunctive normal form} if every concept of the form $C\sqcup D$ only occurs in a CI or in an assertion under a universal role restriction. 

\begin{definition}\label{KBAbductionProblemFull}
Let $\tup{\Kmc,\Psi,\Sig}$ be an abduction problem and $\hyp$ a solution to it satisfying Def.~\ref{KBAbductionProblem}. Then $\hyp = \KB_1 \lor \ldots \lor \KB_{n}$ is a \emph{space of independent explanations} if it is in disjunctive normal form and 
there is no disjunct $\KB_i$ in $\hyp$ such that: %
$$\KB\wedge \KB_i \models \KB_1 \lor \ldots \lor \KB_{i-1} \lor \KB_{i+1} \lor \ldots \lor \KB_{n}.$$
\end{definition}

In the above example, $\hyp_2$ satisfies Def.~\ref{KBAbductionProblemFull} while $\hyp_1$ does not, since $\KB\wedge(C \sqsubseteq D) \models A \sqsubseteq B$. It is worth noting that this condition also constrains hypotheses to those that are consistent with the background KB, i.e., if a hypothesis~$\hyp$ satisfies Def.~\ref{KBAbductionProblemFull} then condition \ref{en:consistent} of Def.~\ref{KBAbductionProblem} will also be satisfied and does not need to be checked separately.

In practice, the entailment tests in Def.~\ref{KBAbductionProblemFull} are done using a DL reasoner 
that checks the satisfiability of 
$\KB\wedge\KB_i\wedge\neg\KB_1\wedge\ldots\neg\KB_{i-1}\wedge\neg\KB_{i+1}
\wedge\ldots\wedge\neg\KB_n$. Unfortunately, current DL reasoners do not support fixpoint operators.
Positive occurrences of greatest fixpoints, and thus also least fixpoints occurring under the 
negation in $\neg\KB_j$, $1\leq j\leq n$, can be simulated using auxiliary concept names:  
for this, we replace $\nu X.C[X]$ by a fresh name $D$ and add the CI $D\sqsubseteq C\rewriteto{D}{X}$.
However, this does not work for positive occurrences of least fixpoints.
Thus, in practice, we cannot detect whether a disjunct with fixpoints is redundant. We therefore 
keep these disjuncts and include them for the redundancy test of the other disjuncts.

\section{Evaluation}
\label{sec:evaluation}

\newcommand{\corpc}{(\texttt{c})\xspace}
\newcommand{\corpi}{(\texttt{i})\xspace}
\newcommand{\KBexp}{(\texttt{K})\xspace}
\newcommand{\ABoxexp}{(\texttt{A})\xspace}
\newcommand{\KBexpA}{(\texttt{K1})\xspace}
\newcommand{\ABoxexpA}{(\texttt{A1})\xspace}
\newcommand{\KBexpB}{(\texttt{K2})\xspace}
\newcommand{\ABoxexpB}{(\texttt{A2})\xspace}
\newcommand{\hypinit}{\texttt{H1}\xspace}
\newcommand{\hypindep}{\texttt{H2}\xspace}

To evaluate our approach, we implemented a prototype in Scala and Java using the 
OWL-API.\footnote{http://owlapi.sourceforge.net/} We added redundancy elimination as 
in~\cite{KOOPMANN2013} to the saturation procedure, and implemented some equivalence-preserving 
syntactical rewritings to optimise the syntactical shape of the result. While we reused some code 
from the forgetting tool \LETHE~\cite{LETHE}, most had to be reimplemented from scratch due to the 
set-of-support strategy, and because the calculus required different data structures, indexing 
and strategies for determining relevant rule applications efficiently. The prototype and 
an explanation on how to repeat the experiment are available 
online.\footnote{\url{https://lat.inf.tu-dresden.de/evaluation-kr2020-dl-abduction/}} 

We created a corpus based on the ``DL Instantiation'' track from the OWL Reasoner 
Competition~2015~\cite{ORE2015}, as it provides a balanced mix of different ontologies containing 
both TBox axioms and ABox assertions.
Each ontology was restricted to its $\ALC$ fragment, where axioms 
such as domain restrictions or disjointness axioms were turned into corresponding CIs. 
We excluded inconsistent ontologies, as well as those of more than 50,000 and less than 100 axioms 
from the resulting set, leaving 202 ontologies. Statistics regarding the corpus are in Table~\ref{corpora-stats}. 

For these ontologies, we created two sets of abduction problems,~\ABoxexp and~\KBexp.
% which we describe in detail in the following. 
Both sets were split into larger and smaller problems to get a 
feeling of the impact of the observation size. The smaller problems contained observations of up to 
10 axioms, the larger problems observations of up to 20 axioms. The exact number of axioms was 
picked from a uniform distribution. This way, we obtained four sets of observations, in the 
following 
denoted \ABoxexpA, \ABoxexpB, \KBexpA and~\KBexpB.

The aim of \ABoxexp was to generate random observations with little assumptions on their shape.
\ABoxexp contained observations including randomly generated ABox assertions. 
For this, we randomly used 2--5 fresh individual names, for which we generated %
the required number of assertions of the form~$A(a)$ and $r(a,b)$ by randomly selecting concept and 
role names from the background ontology. \ABoxexp thus contained ABox abduction problems, but using 
names that may have only little relation to each other.
For \KBexp, we generated less random observations for KB abduction problems that better reflect the 
structure of the ontology, and use names that are related to each 
other. %
To have a realistic mix of CIs and assertions that reflects the typical shape of the 
background ontology, observations in \KBexp were generated by selecting the given number of 
axioms from the background ontology, which were then removed from the background. Since in large 
ontologies, a fully random selection would result in an observation of unrelated axioms, we first 
extracted a subset of at least 100 axioms by combining randomly selected \emph{genuine modules}: 
genuine modules are small subsets of the ontology that contain a given axiom and preserve all 
entailments over the signature of the subset, and thus contain only axioms that in some way interact 
with each 
other~\cite{ATOMIC-DECOMPOSITIONS}. From these subsets of the ontology, which contained between 100 
and 20,979 axioms (median 199), we generated the observations by random selection. 
For each ontology, we created 
150 observations each for the sets \KBexpA and \KBexpB, and 50 observations each for the sets 
\ABoxexpA and~\ABoxexpB. We ignored observations that were already entailed by the 
background ontology, which happened in 3.4\% of cases for \ABoxexp and in 24.9\% of cases 
for~\KBexp.
Signatures for~\ABoxexp and \KBexp were generated by selecting respectively 50\% and 60\% of the 
background ontology's signature, where we made sure at least one name from the observation was not 
an abducible. To reflect the differing relevance of names, each name was chosen with a probability 
proportional to its number of occurrences. For instance, for an ontology describing partonomies, it 
would be unlikely to compute hypotheses without the $\cn{hasPart}$ relation.
Selecting names this way reduced the number of trivial solutions significantly. 
\begin{table}
\small
\centering 
\setlength{\tabcolsep}{1.9mm}
\begin{tabular*}{0.99\linewidth}{@{}cccc|cccc@{}}
\hline
\multicolumn{4}{c|}{TBox size (axioms)} & \multicolumn{4}{c}{ABox size (axioms)}\\
\hline
Min& Max & Mdn & Mean & Min & Max & Mdn & Mean\\
\hline
48 & 36302 & 885 & 3146 & 32 & 42429 & 1424 & 4610\\
\hline
\end{tabular*}
\caption{Characteristics of the experimental corpus (Mdn: median).}
\label{corpora-stats}
\end{table}

\begin{table}%
\small
\centering 
\setlength{\tabcolsep}{0.65mm}
\begin{tabular*}{0.99\linewidth}{@{}l@{}ccccc|ccccc@{}}
\hline
& \multicolumn{5}{c|}{Initial Hypothesis (s.)} & \multicolumn{5}{c}{Indep. Explanations 
(s.)}\\
\hline
& Min & Mdn & Mean & P90 & Max & Min & Mdn & Mean & P90 & Max\\
\hline
\KBexpA & 0.2 & 0.6 & 2.4 & 1.8 & 293.2 & 0.0 & 0.0 & 1.2 & 0.8 & 292.5\\
\KBexpB & 0.2 & 0.4 & 2.5 & 2.1 & 294.8 & 0.0 & 0.0 & 1.1 & 0.6 & 293.4\\
\ABoxexpA & 0.2 & 1.9 & 16.2 & 51.7 & 292.0 & 0.0 & 0.1 & 4.6 & 2.2 & 282.1\\
\ABoxexpB & 0.2 & 1.6 & 16.9 & 49.6 & 296.9 & 0.0 & 0.1 & 4.9 & 2.5 & 293.1\\
\hline
\end{tabular*}
\caption{Running times (Mdn: median, P90: 90th percentile).}
\label{results-time}
\end{table}
\warren{Need better titles in Table 2.}

The experiments were run on an Intel Core i5-4590 CPU machine with 3.30GHz and 32 GB RAM, using 
Debian/GNU Linux 9 and OpenJDK 11.0.5.
For each abduction problem, the timeout was set to 5 minutes. %
The hypotheses obtained after applying the method in Sect.~\ref{sec:forgetting-method} 
(computing optimal hypotheses) and those obtained after additionally applying the method in 
Sect.~\ref{hypothesis-filtering} (removing redundancies) are referred to as \hypinit and \hypindep 
respectively. 
Computation times are shown in Table \ref{results-time}. 
The success rate (no timeout nor out-of-memory exception)
for \KBexpA was 96.3\% for \hypinit and 95.3\% for \hypindep.
For \ABoxexpA, success rates were 91.3\% for \hypinit and 90.1\% for \hypindep. 
For observations of size up to 20, the rates were very similar: \KBexpB/\ABoxexpB got 
success rates of 96.4\%/91.1\% for \hypinit and 95.4\%/89.5\% for~\hypindep. The size of the 
observation thus had only very little impact.
Since we selected the signatures randomly, in a lot of cases they did not fit to the observation 
and no hypothesis according to Def.~\ref{KBAbductionProblem} existed. This was the case for 
64.5\% and 67.4\% of the observations in \ABoxexpA and in \ABoxexpB, and for 72.6\% and 76.1\% of 
the observations in \KBexpA and \KBexpB.
In the following, we focus on the remaining cases.

Table \ref{results-hypotheses} shows statistics regarding the size of computed hypotheses.
Eliminating redundant disjuncts for \hypindep often reduced the sizes of the hypotheses produced 
for both experiment \KBexp and experiment \ABoxexp. 
Table \ref{results-disjunctive} shows statistics about the disjuncts in each hypothesis. For 
\KBexpA/\KBexpB, disjunctions were required to represent \hypinit and \hypindep in 18.0\%/20.9\% 
and 7.1\%/7.7\% of cases respectively. For \ABoxexpA/\ABoxexpB, disjunctions where needed much 
more often: here the values were 47.2\%/43.8\% and 36.8\%/32.4\%. This can be explained by the fact 
that 
disjunctive hypotheses are more easy to obtain from ABox assertions than from CIs, and that many 
observations in \KBexp contained only CIs.
Though our method may introduce inverse roles, nominals and fixpoint operators, this was only 
observed in the minority of cases: inverse roles and nominals were needed in 16.4\%/16.6\% of cases 
for \KBexp and in 1.5\%/1.3\% of cases for \ABoxexp. Fixpoint expressions where needed in 
3.1\%/5.3\% of cases for \KBexp and 0.8\%/1.3\% of cases for~\ABoxexp. 

\begin{table}
\small
\centering 
\setlength{\tabcolsep}{0.85mm}
\begin{tabular*}{0.99\linewidth}{@{}lccccc|ccccc@{}}
\hline
& \multicolumn{5}{c|}{Number of TBox axioms} & \multicolumn{5}{c}{Number of ABox axioms}\\
\hline
& Min & Mdn & Mean & P90 & Max & Min & Mdn & Mean & P90 & Max\\
  \hline
  \multicolumn{11}{c}{\KBexpA}\\
  \hline
  \hypinit &
  0 & 2 & 5.6 & 7 & 846 & 
  0 & 3 & 18.6 & 9 & 10294 \\
  \hypindep & 
  0 & 1 & 2.3 & 5 & 105 &
  0 & 3 & 6.4 & 8 & 4374\\
  \hline
  \multicolumn{11}{c}{\KBexpB}\\
  \hline
  \hypinit &
  0 & 2 & 8.9 & 14 & 1848 & 
  0 & 6 & 42.9 & 19 & 28817 \\
  \hypindep & 
  0 & 2 & 3.9 & 8 & 1848 &
  0 & 5 & 9.0 & 16 & 1212\\
  \hline
  \multicolumn{11}{c}{\ABoxexpA}\\
  \hline
  \hypinit &
  0 & 0 & 3.9 & 0 & 5762 & 
  0 & 9 & 30.5 & 51 & 8296 \\
  \hypindep & 
  0 & 0 & 0.0 & 0 & 5 &
  0 & 8 & 17.1 & 36 & 1176\\
  \hline
  \multicolumn{11}{c}{\ABoxexpB}\\
  \hline
  \hypinit &
  0 & 0 & 2.4 & 0 & 2646 & 
  0 & 14 & 52.2 & 80 & 21146 \\
  \hypindep & 
  0 & 0 & 0.0 & 0 & 12 &
  0 & 12 & 24.2 & 52 & 1536\\
\hline
\end{tabular*}
\caption{TBox and ABox axioms in the computed hypotheses (Mdn: median, P90: 90th percentile).}
\label{results-hypotheses}
\end{table}

\begin{table}
\small
\centering 
\setlength{\tabcolsep}{1.5mm}
\begin{tabular*}{0.99\linewidth}{@{}lcccc|cccc@{}}
\hline
& Mdn. & Mean & P90. & Max & Mdn. & Mean & P90. & Max\\
\hline
& \multicolumn{4}{c|}{\KBexpA} & \multicolumn{4}{c}{\ABoxexpA}\\
\hline
\hypinit & 1 & 7.1 & 3 & 4751 
         & 2 & 8.3 & 10 & 4184\\
\hypindep & 1 & 1.8 & 1 & 729 
          & 2 & 3.7 & 8 & 343 \\
\hline
& \multicolumn{4}{c|}{\KBexpB} & \multicolumn{4}{c}{\ABoxexpB}\\
\hline
\hypinit & 1 & 15.3 & 4 & 13126 
         & 2 & 11.9 & 10 & 7796\\
\hypindep & 1 & 1.8 & 1 & 243 
          & 1 & 3.5 & 7 & 256 \\
\hline
\end{tabular*}
\caption{Number of disjuncts in the computed hypotheses (Mdn: median, P90: 90th percentile).}
\label{results-disjunctive}
\end{table}

\section{Conclusion}
\label{sec:conclusion}

\newcommand{\ALCI}{\ensuremath{\mathcal{ALCI}}\xspace}

We presented the first general method for signature-based abduction on \ALC KBs. At its center lies 
a new resolution-based calculus for Boolean KBs, which is used in combination with a set-of-support 
strategy.
We combined it with a simplification technique ensuring the generation of independent explanations 
and evaluated the overall technique in practice on realistic benchmarks.

While the theoretical complexity of our problem remains open,
we conjecture that the situation is the 
same as for uniform interpolation~\cite{FOUNDATIONS-EXPRESSIVE-UI}, and that optimal, fixpoint-free 
solutions can be of size triple-exponential in the input.
A natural next step is to allow nominals and inverse roles not only in the output, but also in the 
input, which especially for nominals might prove challenging. 
Finally, we plan to use
our approach for abduction-related tasks %
such as gentle repairs~\cite{GENTLE-REPAIRS} and 
induction with iterative refinement~\cite{LehmannHitzler10}.

\subsubsection{Acknowledgements}

Patrick Koopmann and Sophie Tourret are supported by DFG grant 389793660 of TRR 248.
Warren Del-Pinto is a recipient of a RCUK Scholarship.

%
%\bibliographystyle{kr}
%\bibliography{bibliography-abduction-forgetting}

\ifdefined\showAppendix

\newpage

\appendix

\fi
\end{document}